\documentclass{article}
\usepackage{iclr2027_conference,times}
\usepackage{amsmath,amssymb,amsthm,mathtools,mathrsfs}
\usepackage{graphicx}
\usepackage{booktabs}
\newcommand{\C}{\mathbb{C}}
\newcommand{\E}{\mathbb{E}}
\newcommand{\R}{\mathbb{R}}
\newcommand{\calL}{\mathcal{L}}

\newcommand{\iter}{\mathcal{I}}
\newcommand{\calB}{\mathcal{B}}
\newcommand{\flow}{\Phi}
\newcommand{\target}{\mathcal{Y}}
\newcommand{\Normal}{\mathcal{N}}
\newcommand{\Fourier}{\mathscr{F}}
\newcommand{\U}{\operatorname{U}}
\newcommand{\sg}{\operatorname{sg}}
\newcommand{\supp}{\operatorname{supp}}
\renewcommand{\Re}{\operatorname{Re}}
\newcommand{\dd}{\mathrm{d}}
\newcommand\norm[2][1]{\left\|#2\right\|_{#1}}
\newcommand\abs[1]{\left|#1\right|}
\newcommand\where[2]{\left.#1\right|_{#2}}

\usepackage{hyperref}
\usepackage{thmtools,thm-restate}
\usepackage{url}
\usepackage[ruled,lined]{algorithm2e}

\theoremstyle{plain}
\newtheorem{theorem}{Theorem}[section]

\newtheorem{proposition}[theorem]{Proposition}

\def\theHtheorem{\thesection.\arabic{theorem}}

\hypersetup{colorlinks=true,linkcolor=blue,citecolor=blue,urlcolor=blue,
    pdftitle={Score-Calibrated Flow},pdfauthor={Zeyang Li, Yunan Wang, Risheek Garrepalli, Mohammad Ghavamzadeh, Navid Azizan}}

\title{Score-Calibrated Flow for Sampling from Unnormalized Densities with Applications to Generative Online Reinforcement Learning}
\author{Zeyang Li\textsuperscript{1}, Yunan Wang\textsuperscript{1}, Risheek Garrepalli\textsuperscript{2}, Mohammad Ghavamzadeh\textsuperscript{2}, Navid Azizan\textsuperscript{1} \\
\textsuperscript{1}Massachusetts Institute of Technology \quad \textsuperscript{2}Qualcomm AI Research}
\iclrfinalcopy

\makeatletter
\newcommand{\restatethmlike}[2]{%
    \begingroup
    \renewcommand{\thetheorem}{\csname the#1\endcsname}%
    \let\c@theorem\c@thmt@dummyctr
    \let\theHtheorem\theHthmt@dummyctr
    #2*%
    \endgroup
}
\newcommand{\titlenote}[1]{%
    \begingroup
    \renewcommand{\thefootnote}{}%
    \long\def\@makefntext##1{\noindent##1}%
    \footnotetext{#1}%
    \endgroup
}
\makeatother
\newcommand{\restateproposition}[1]{\restatethmlike{proposition}{#1}}
\newcommand{\restatetheorem}[1]{\restatethmlike{theorem}{#1}}
\newcommand{\restatecorollary}[1]{\restatethmlike{corollary}{#1}}

\begin{document}
{\setlength{\tabcolsep}{0pt}\maketitle}
\titlenote{Correspondence to Zeyang Li (\texttt{zeyang@mit.edu}) and Yunan Wang (\texttt{wangyn25@mit.edu}).}
\ificlrfinal
  \lhead{Preprint}
\else
  \lhead{Under review}
\fi

\begin{abstract}
Diffusion and flow models provide expressive policy classes for online reinforcement learning (RL), enabling multimodal behaviors and improved performance. However, training these policies remains challenging: the critic specifies the desired policy as an unnormalized Boltzmann density but does not provide direct samples from it. Many existing methods rely on importance sampling to construct training signals, which can suffer from high variance, increasing computational cost and destabilizing training.
In this paper, we propose \emph{Score-Calibrated Flow} (SCF), a simple and efficient algorithm for training generative models to sample from unnormalized densities without importance sampling or backpropagation through the sampling trajectory. Our key idea is to learn the desired flow by enforcing self-consistency, bypassing the challenging estimation of the target posterior mean. By jointly exploiting the prescribed target score and the structure of flow matching, we establish these self-consistency requirements as \emph{score-calibrated optimality conditions}, first for the terminal density and then for the trainable velocity field. We prove that their unique solutions are, respectively, the target density and the ideal flow model that conditional flow matching (CFM) would recover if target samples were available.
We formulate the velocity condition as a fixed-point equation and exploit its conditional-expectation structure to construct a stop-gradient objective for enforcing it. The resulting training procedure retains the scalable sample--interpolate--regress structure of CFM despite the absence of target samples, using endpoints generated by the current flow. For online RL, the critic gradient supplies the target score at the generated actions, yielding a direct approach to actor training.
Experiments on RL benchmarks demonstrate that SCF matches or improves upon state-of-the-art generative-policy baselines, while substantially reducing training time.
\end{abstract}

\section{Introduction}
\label{sec:introduction}

Diffusion and flow models learn expressive distributions by transforming initial noise through stochastic or deterministic dynamics~\citep{sohl2015deep,ho2020denoising,lipman2023flow}. They are appealing as policy classes since several distinct actions may be valuable in the same state, whereas a restrictive distribution may fail to represent these alternatives. Greater policy expressiveness can also translate into improved performance. These advantages have motivated the use of generative policies in imitation learning~\citep{chi2025diffusion} and offline reinforcement learning (RL)~\citep{wang2023diffusion}. In these settings, expert demonstrations or offline datasets provide the samples needed in the standard training recipe for generative models. In online RL, however, direct samples from the target density that policy improvement asks the actor to learn are unavailable.

The difficulty is particularly clear in maximum-entropy RL~\citep{haarnoja2018soft}. At a given state $s$, soft policy improvement targets the Boltzmann policy $\pi^Q(a|s)\propto\exp(\tau Q(s,a))$, where $Q$ is the critic and $\tau>0$ is the inverse temperature. Soft actor-critic (SAC) approximates this target within a Gaussian policy family using a mode-seeking objective. A diffusion or flow policy offers a richer representation, but standard training methods such as conditional flow matching (CFM) require samples from the Boltzmann policy, whereas the critic provides only evaluations of $Q(s,a)$ and its action gradient $\nabla_a Q(s,a)$. More broadly, the challenge is \emph{learning to sample from unnormalized densities}, rather than learning from target samples as in standard generative training.

Existing methods for generative online RL can be grouped into four categories (see Appendix~\ref{app:related_work} for a detailed discussion). \emph{Direct backpropagation}~\citep{wang2024diffusion,lv2025flow,celik2025dime} optimizes critic-based objectives by differentiating through the entire sampling trajectory, which incurs substantial computational cost and destabilizes training. \emph{Value-guided sampling}~\citep{psenka2024learning,wei2026scalable} uses critic information to guide sampling and exploration through stochastic or deterministic dynamics. \emph{Weighted regression}~\citep{ding2024diffusion,fan2025online,ma2025efficient} trains the policy on candidate actions using value-dependent weights, adding sampling overhead and making the training target depend on the proposal distribution and weighting rule. These approaches rely on surrogate objectives or sampling approximations that can introduce bias relative to exact soft policy improvement.
In contrast, \emph{Boltzmann targeting}~\citep{dong2025maximum,ma2025efficient,jain2025sampling,li2026reverse} derives training signals directly from the Boltzmann density, recovering the desired policy at the ideal population optimum. In practice, however, the required posterior means are typically estimated through multi-particle importance sampling.
This incurs computational overhead at each update and can produce high-variance estimates that destabilize training.

In this paper, we propose \emph{Score-Calibrated Flow} (SCF), a simple and efficient algorithm for learning to sample from unnormalized densities. Our key idea is to learn the desired flow by enforcing self-consistency, bypassing the challenging estimation of the target posterior mean. By jointly exploiting the prescribed target score and the structure of flow matching, we establish these self-consistency requirements as \emph{score-calibrated optimality conditions}. We first derive conditions that characterize the target density, then develop corresponding conditions for the trainable velocity field.

The velocity condition takes the form of a fixed-point equation whose conditional-expectation structure enables training with self-generated samples. Specifically, SCF generates an endpoint with the current flow, interpolates it with independent source noise, and constructs a regression target whose conditional mean equals the right-hand side of the fixed-point equation. This retains the scalable sample--interpolate--regress structure of CFM despite the absence of target samples, without importance sampling or backpropagation through the sampling trajectory.

Our contributions are summarized as follows.
\begin{enumerate}
    \item \textbf{Score-calibrated optimality conditions.}
    We establish necessary and sufficient self-consistency conditions for both the terminal density and the trainable velocity field. Their unique solutions are, respectively, the target density and the ideal flow model that CFM would recover if target samples were available.

    \item \textbf{Scalable training from self-generated samples.}
    We construct a sample-wise regression target whose conditional mean equals the score-calibration operator applied to the current velocity. We use this target in a stop-gradient objective to enforce the velocity condition. Each target requires only one self-generated endpoint and one target-score evaluation, avoiding importance sampling and backpropagation through the sampling trajectory.

    \item \textbf{Principled design and strong RL performance.}
    We derive coefficient schedules that cancel source-end singular factors and handle bounded actions through an exact change of variables that preserves the intended target density. Experiments on online RL benchmarks demonstrate that SCF matches or improves upon state-of-the-art generative-policy baselines, with substantial savings in training time.
\end{enumerate}

\section{Background}
\label{sec:background}

\subsection{Sampling from Unnormalized Densities}
\label{sec:sampling-problem}

We consider sampling from a target density on $\R^d$ defined by a continuously differentiable potential $Q:\R^d\to\R$ and an inverse temperature $\tau>0$:
\begin{equation}
    \pi(x)=\frac{\exp(\tau Q(x))}{\displaystyle\int_{\R^d}\exp(\tau Q(z))\,\dd z}.
    \label{eq:target-density}
\end{equation}
The normalizing integral is assumed finite, and the target score is $\nabla\log\pi(x)=\tau\nabla Q(x)$. Our goal is to learn a sampler for $\pi$ given access to evaluations of $Q(x)$ and its gradient $\nabla Q(x)$, but not to the normalizing constant or direct target samples.

We first develop our method in this general setting, where $x\in\R^d$ denotes a point in data space. We later specialize it to RL, reusing $Q$ to denote the state--action value function $Q(s,a)$. At each fixed state $s$, this function serves as the sampling potential over actions $a$.

\subsection{Flow Matching}
\label{sec:cfm-background}

Flow matching learns a time-dependent velocity field that transports a source density $p_0$ to a target density $p_1$~\citep{lipman2023flow}. We adopt the standard setting of a Gaussian source $p_0=\Normal(0,I)$, an independent coupling of endpoints $X_0\sim p_0$ and $X_1\sim p_1$, and linear interpolation:
\begin{equation}
    X_t=\alpha_tX_1+\beta_tX_0,\qquad t\in[0,1].
    \label{eq:interpolant}
\end{equation}
The interpolation schedules $\alpha,\beta\in C^1([0,1])$ satisfy $\alpha_0=\beta_1=0$, $\alpha_1=\beta_0=1$, and $\alpha_t,\beta_t>0$ for $t\in(0,1)$.

For $t\in(0,1)$, one-sided conditioning on $X_1=x_1$ gives the conditional probability path
\begin{equation}
    p_{t|1}(x_t|x_1)=\Normal(x_t;\alpha_tx_1,\beta_t^2I),
    \label{eq:conditional-density}
\end{equation}
with associated conditional velocity
\begin{equation}
    v_{t|1}(x_t|x_1)
    =\dot\alpha_tx_1+\dot\beta_t\frac{x_t-\alpha_tx_1}{\beta_t}.
    \label{eq:conditional-velocity}
\end{equation}
Conditional flow matching (CFM) uses this velocity as a regression target to train a parameterized field $v^\theta$:
\begin{equation*}
    \calL_{\mathrm{CFM}}(\theta)
    :=\E_{t\sim\U(0,1),\,X_1\sim p_1,\,X_t\sim p_{t|1}(\cdot|X_1)}
    \left[\norm[2]{v_t^\theta(X_t)-v_{t|1}(X_t|X_1)}^2\right].
\end{equation*}
The population minimizer is the conditional mean of the regression target given $(t,X_t)$.

To generate samples with the learned field $v^\theta$, we solve the ordinary differential equation (ODE)
\begin{equation*}
    \frac{\dd Y_t}{\dd t}=v_t^\theta(Y_t),\qquad Y_0\sim p_0,\qquad t\in[0,1].
\end{equation*}
Here $Y_t$ denotes the sampling trajectory, while $X_0,X_t,X_1$ are reserved for the interpolant used to construct training pairs. With perfect CFM training, the terminal state $Y_1$ has density $p_1$. This sample--interpolate--regress construction requires direct target samples $X_1\sim p_1$, which are unavailable in the unnormalized-density setting considered in this paper.

\subsection{Maximum-Entropy Reinforcement Learning}
\label{sec:maxent-rl}

We consider a Markov decision process (MDP) $(\mathcal S,\mathcal A,P,r,d_0,\gamma)$, with state and action spaces $\mathcal S,\mathcal A$, transition kernel $P$, reward function $r$, initial state distribution $d_0$, and discount factor $\gamma\in(0,1)$. In the RL setting, $\pi(a|s)$ denotes a policy. Maximum-entropy RL~\citep{haarnoja2018soft} augments the expected discounted return with policy entropy $\mathcal H(\pi(\cdot|s))=-\E_{a\sim\pi(\cdot|s)}[\log\pi(a|s)]$, weighted by $\frac{1}{\tau}$.

For a policy $\pi$, the regularized self-consistency operator is given by
\begin{equation*}
    (\calB^\pi Q)(s,a)
    =r(s,a)+\gamma\E_{s'\sim P(\cdot|s,a)}
    \left[\E_{a'\sim\pi(\cdot|s')}[Q(s',a')]
    +\frac{1}{\tau}\mathcal H(\pi(\cdot|s'))\right].
\end{equation*}
The Bellman operator is given by $(\calB Q)(s,a)=\max_\pi\{(\calB^\pi Q)(s,a)\}$. The soft Q-function $Q^\pi$ and its optimal counterpart $Q^*$ satisfy $Q^\pi=\calB^\pi Q^\pi$ and $Q^*=\calB Q^*$.

Holding $Q=Q^\pi$ fixed, soft policy improvement maximizes $\E_{a\sim\pi(\cdot|s)}[Q(s,a)]+\frac{1}{\tau}\mathcal H(\pi(\cdot|s))$ at each state, yielding the critic-prescribed Boltzmann policy
\begin{equation}
    \pi^Q(a|s)
    =\frac{\exp(\tau Q(s,a))}
    {\displaystyle\int_{\mathcal A}\exp(\tau Q(s,a'))\,\dd a'}.
    \label{eq:boltzmann-policy}
\end{equation}
Thus, soft policy improvement requires sampling from an unnormalized density specified by the critic.

\section{Score-Calibrated Flow}
\label{sec:scf}

In this section, we develop Score-Calibrated Flow (SCF) for learning generative models to sample from unnormalized densities.
The key paradigm shift enabling SCF is to reinterpret the flow matching framework as a one-to-one correspondence between terminal densities and their associated ideal flow models, rather than solely as a procedure for learning from target samples.
Each terminal density determines the population minimizer of the CFM loss, which is the velocity field that CFM would learn if samples from that density were available. Conversely, pushing the source density through the resulting flow induces the terminal density.

We exploit the two directions of this correspondence to develop our optimality conditions. First, we view the flow model as a function of its terminal density and derive a score-calibrated condition that characterizes the target density through the flow it induces. We then let the trainable velocity field determine its own terminal density through the pushforward map, yielding a self-consistency condition directly on the velocity. This condition takes the form of a fixed-point equation whose conditional-expectation structure enables scalable regression using self-generated samples.
Finally, we develop principled design choices to address practical challenges, including numerical singularities and bounded domains, and demonstrate how to apply SCF to generative online RL.

\subsection{Score-Calibrated Optimality Condition for Terminal Density}
\label{sec:target-characterization}
\label{sec:optimality-equations}

We retain the standard choices from Section~\ref{sec:cfm-background}: a Gaussian source $p_0=\Normal(0,I)$, independent coupling, and the linear interpolant \eqref{eq:interpolant}. Under these choices, the conditional probability path $p_{t|1}$ and conditional velocity $v_{t|1}$ in \eqref{eq:conditional-density}--\eqref{eq:conditional-velocity} are also fixed.

In contrast, the marginal probability path and endpoint posterior depend on the terminal density. For a candidate density $\rho$, we take $X_1\sim\rho$ and use a superscript to make this dependence explicit:
\begin{equation*}
    p_t^\rho(x_t):=\int_{\R^d}\rho(x_1)p_{t|1}(x_t|x_1)\,\dd x_1,
    \qquad
    p_{1|t}^\rho(x_1|x_t):=\frac{\rho(x_1)p_{t|1}(x_t|x_1)}{p_t^\rho(x_t)}.
\end{equation*}
The latter is the posterior density of $X_1$ given the interpolant $X_t=x_t$. It determines the ideal flow model associated with $\rho$:
\begin{equation}
    v_t^\rho(x_t):=\E_{X_1\sim p_{1|t}^\rho(\cdot|x_t)}[v_{t|1}(x_t|X_1)].
    \label{eq:population-velocity}
\end{equation}
This is the velocity field that CFM would learn from samples of $\rho$. Its flow generates the marginal path $p_t^\rho$ and transports the source density to $\rho$. Thus, the posterior, marginal probability path, and marginal velocity vary together with the terminal density, while the conditional probability path and conditional velocity remain unchanged.

To relate the flow model to the terminal score, define
\begin{equation}
    \kappa_t:=\frac{\beta_t}{\alpha_t}(\dot\alpha_t\beta_t-\alpha_t\dot\beta_t),
    \qquad t\in(0,1).
    \label{eq:kappa}
\end{equation}
We assume $\kappa_t>0$, equivalently $\frac{\dd}{\dd t}\log\left(\frac{\alpha_t}{\beta_t}\right)>0$. This requires the ratio $\frac{\alpha_t}{\beta_t}$ to increase along the noise-to-data path, a condition commonly satisfied by standard diffusion and flow schedules. The following identity expresses $v_t^\rho$ in terms of $\nabla\log\rho$.

\begin{restatable}{proposition}{scfScoreVelocity}
\label{prop:score-velocity}
Let $\rho$ be a positive, normalized $C^1$ density satisfying the regularity condition \eqref{eq:posterior-moment-condition} in Appendix~\ref{app:regularity}. For every $t\in(0,1)$ and $x_t\in\R^d$,
\begin{equation}
    v_t^\rho(x_t)=\frac{\dot\alpha_t}{\alpha_t}x_t
    +\frac{\kappa_t}{\alpha_t}\E_{X_1\sim p_{1|t}^\rho(\cdot|x_t)}[\nabla\log\rho(X_1)].
    \label{eq:score-velocity}
\end{equation}
\end{restatable}

Applying \eqref{eq:score-velocity} to the target density $\pi$, whose score is $\tau\nabla Q$, gives
\begin{equation}
    v_t^\pi(x_t)=\frac{\dot\alpha_t}{\alpha_t}x_t
    +\frac{\kappa_t}{\alpha_t}\E_{X_1\sim p_{1|t}^\pi(\cdot|x_t)}[\tau\nabla Q(X_1)].
    \label{eq:target-identity}
\end{equation}
We now take a different perspective on this equation: rather than using the target density to describe its associated flow, we seek to characterize the target density through the flow it induces under the CFM construction. Specifically, we treat the terminal density $\rho$ as the unknown, allowing it to determine both the marginal velocity $v^\rho$ and the posterior $p_{1|t}^\rho$, while keeping the prescribed target score $\tau\nabla Q$ fixed. The following theorem establishes the resulting self-consistency equation as a score-calibrated optimality condition for the terminal density.

\begin{restatable}{theorem}{scfTargetCharacterization}
\label{thm:target-characterization}
Let $\pi$ be the target density defined in \eqref{eq:target-density}. Let $\rho$ be a positive, normalized $C^1$ density satisfying regularity conditions \eqref{eq:posterior-moment-condition} and \eqref{eq:nabla-Q-integrability-on-rho-posterior} in Appendix~\ref{app:regularity}. Define the density optimality condition
\begin{equation}
    v_t^\rho(x_t)=\frac{\dot\alpha_t}{\alpha_t}x_t
    +\frac{\kappa_t}{\alpha_t}\E_{X_1\sim p_{1|t}^\rho(\cdot|x_t)}[\tau\nabla Q(X_1)]\quad\text{for every }x_t\in\R^d.
    \label{eq:calibration-equation}
\end{equation}
Then
\begin{equation*}
    \rho=\pi
    \quad\Leftrightarrow\quad
    \text{\eqref{eq:calibration-equation} holds for every }t\in(0,1)
    \quad\Leftrightarrow\quad
    \text{\eqref{eq:calibration-equation} holds for some }t\in(0,1).
\end{equation*}
\end{restatable}

The sufficiency of $\rho=\pi$ follows directly from \eqref{eq:target-identity}, which holds at every interior time. The necessity is nontrivial: satisfying \eqref{eq:calibration-equation} for every $x_t$ at even a single $t\in(0,1)$ forces $\rho=\pi$. Thus, requiring the marginal velocity $v^\rho$ to agree with the velocity reconstructed from the prescribed target score $\tau\nabla Q$ under the posterior $p_{1|t}^\rho$ uniquely identifies $\rho=\pi$. Enforcing this relation for every $x_t$ at a single interior time is equivalent to enforcing it throughout the interpolation path.

\paragraph{Generalized optimality condition for terminal density.}
The preceding condition identifies the desired density through its associated flow model. We next generalize its form without changing its solution, introducing coefficients that will later allow us to address numerical singularities and adjust the sample-wise training signal.

First, for any $\lambda_t\ne0$, we can equivalently write \eqref{eq:calibration-equation} as
\begin{equation}
    v_t^\rho(x_t)=(1-\lambda_t)v_t^\rho(x_t)
    +\lambda_t\left(\frac{\dot\alpha_t}{\alpha_t}x_t
    +\frac{\kappa_t}{\alpha_t}
    \E_{X_1\sim p_{1|t}^\rho(\cdot|x_t)}[\tau\nabla Q(X_1)]\right).
    \label{eq:relaxed-population-update}
\end{equation}
The coefficient $\lambda_t$ controls the weight of the score-based term while preserving $\rho=\pi$ as the unique solution. Next, recall that the marginal velocity is the conditional expectation of $v_{t|1}$. We can therefore express part of the remaining velocity term through this relationship: for any $\zeta_t\in\R$,
\begin{equation}
    (1-\lambda_t)v_t^\rho(x_t)
    =(1-\lambda_t-\zeta_t)v_t^\rho(x_t)
    +\zeta_t\E_{X_1\sim p_{1|t}^\rho(\cdot|x_t)}
    [v_{t|1}(x_t|X_1)].
    \label{eq:velocity-decomposition}
\end{equation}
Substituting \eqref{eq:velocity-decomposition} into \eqref{eq:relaxed-population-update} gives the following generalized optimality condition.

\begin{restatable}{proposition}{scfAffineCalibration}
\label{prop:affine-calibration}
Adopt the notation and assumptions of Theorem~\ref{thm:target-characterization}. Let $\lambda_t\ne0$ and $\zeta_t\in\R$. Define the generalized density optimality condition
\begin{align}
    v_t^\rho(x_t)=&(1-\lambda_t-\zeta_t)v_t^\rho(x_t)
    +\zeta_t\E_{X_1\sim p_{1|t}^\rho(\cdot|x_t)}[v_{t|1}(x_t|X_1)]\notag\\
    &+\lambda_t\left(\frac{\dot\alpha_t}{\alpha_t}x_t
    +\frac{\kappa_t}{\alpha_t}\E_{X_1\sim p_{1|t}^\rho(\cdot|x_t)}[\tau\nabla Q(X_1)]\right)\quad\text{for every }x_t\in\R^d.
    \label{eq:generalized-optimality}
\end{align}
Then
\begin{equation*}
    \rho=\pi
    \quad\Leftrightarrow\quad
    \text{\eqref{eq:generalized-optimality} holds for every }t\in(0,1)
    \quad\Leftrightarrow\quad
    \text{\eqref{eq:generalized-optimality} holds for some }t\in(0,1).
\end{equation*}
\end{restatable}

The generalized condition retains the same unique solution while providing flexibility in the eventual training construction. The coefficient $\lambda_t$ controls the strength of score calibration, while $\zeta_t$ determines how much of the marginal velocity is represented through the conditional velocity.

The optimality conditions for $\rho$ provide an important first step by characterizing the desired terminal density. However, the object we train is the velocity field, not the density itself. To obtain an optimality condition directly on the velocity, we exploit the reverse direction of the correspondence: a velocity field determines its terminal density through the pushforward of the source. Using this generated density to determine the posterior leads to a self-consistency condition on the trainable velocity.

\subsection{Score-Calibrated Optimality Condition for Velocity}
\label{sec:velocity-optimality}

For a velocity field $v$, let $\flow_{0\to t}^v$ be the flow map of $\dot Y_t=v_t(Y_t)$ and define its terminal density by
\begin{equation*}
    \rho^v:=(\flow_{0\to1}^v)_\#p_0,
\end{equation*}
where $\#$ denotes pushforward. The regularity requirements for well-posedness of this definition are specified in Appendix~\ref{app:regularity}.

We now seek a velocity field that satisfies the score-based relation in \eqref{eq:calibration-equation}, with the posterior determined by its own terminal density. To express this self-consistency requirement, define the \emph{score-calibration operator} by
\begin{equation}
    (\iter v)_t(x_t):=\frac{\dot\alpha_t}{\alpha_t}x_t
    +\frac{\kappa_t}{\alpha_t}\E_{X_1\sim p_{1|t}^{\rho^v}(\cdot|x_t)}[\tau\nabla Q(X_1)]
    \label{eq:score-calibration-operator}
\end{equation}
and consider the fixed-point equation $v=\iter v$. The same field supplies the velocity on the left and generates the terminal density used to form the posterior on the right. The desired field $v^\pi$ satisfies this equation by \eqref{eq:target-identity}. The following theorem establishes that this condition is both necessary and sufficient to identify the optimal velocity.

\begin{restatable}{theorem}{scfVelocityOptimality}
\label{thm:velocity-optimality}
Let $\pi$ be the target density defined in \eqref{eq:target-density}, and assume that the regularity condition \eqref{eq:posterior-moment-condition} holds for $\rho=\pi$. Assume that $v^\pi$ and $v$ are admissible, as specified in Appendix~\ref{app:regularity}. Assume that $v$ satisfies \eqref{eq:source-growth} and \eqref{eq:weighted-exponential-moment}. Then, $v=\iter v$ holds on $(0,1)\times\R^d$ if and only if $\rho^v=\pi$ and $v=v^\pi$.
\end{restatable}

The nontrivial implication is that self-consistency forces both $\rho^v=\pi$ and $v=v^\pi$. Thus, the condition does not merely identify the desired terminal density; it uniquely determines the ideal flow model that CFM would learn if target samples were available. This completes the transition from an optimality condition on the density to one directly on the trainable velocity.

\paragraph{Generalized optimality condition for velocity.}
We next extend the two-coefficient construction in \eqref{eq:generalized-optimality} to the velocity formulation. The explicit marginal-velocity term is replaced by the trainable field $v$, while its generated density $\rho^v$ determines both posterior expectations. This gives the \emph{generalized score-calibration operator}
\begin{align}
    (\iter_{\lambda,\zeta}v)_t(x_t):=&(1-\lambda_t-\zeta_t)v_t(x_t)
    +\zeta_t\E_{X_1\sim p_{1|t}^{\rho^v}(\cdot|x_t)}[v_{t|1}(x_t|X_1)]\notag\\
    &+\lambda_t\left(\frac{\dot\alpha_t}{\alpha_t}x_t
    +\frac{\kappa_t}{\alpha_t}\E_{X_1\sim p_{1|t}^{\rho^v}(\cdot|x_t)}[\tau\nabla Q(X_1)]\right).
    \label{eq:generalized-score-calibration-operator}
\end{align}
The corresponding optimality condition is $v=\iter_{\lambda,\zeta}v$. It combines the current velocity, the conditional-velocity expectation, and the target-score expression, with coefficients $\lambda_t$ and $\zeta_t$ controlling their contributions. Setting $\lambda_t=1$ and $\zeta_t=0$ recovers the original condition $v=\iter v$. The following theorem gives a coefficient condition under which the generalized velocity condition also uniquely identifies the desired flow model.

\begin{restatable}{theorem}{scfGeneralizedOptimality}
\label{thm:generalized-optimality}
Assume the hypotheses of Theorem~\ref{thm:velocity-optimality} on $v$ and $\pi$, and let $t_0$ be as in \eqref{eq:source-growth}. Let $\lambda_t>0$ and $\zeta_t\ge0$ be measurable, finite schedules on $(0,1)$ satisfying
\begin{equation}
    \int_0^{t_0}\frac{\lambda_t}{\lambda_t+\zeta_t}\frac{\kappa_t}{\beta_t^2}\,\dd t=\infty.
    \label{eq:generalized-uniqueness-condition}
\end{equation}
Then $v=\iter_{\lambda,\zeta}v$ holds on $(0,1)\times\R^d$ if and only if $\rho^v=\pi$ and $v=v^\pi$.
\end{restatable}

Under the conditions of Theorem~\ref{thm:generalized-optimality}, the original and generalized velocity equations have the same unique solution $v^\pi$. The coefficients therefore provide freedom to weight the conditional-velocity and target-score terms while preserving the desired flow model.

To clarify the relationship with the density formulation, consider a field $v=v^\rho$ obtained from the CFM construction. Its flow generates the terminal density $\rho$, so both operators use the posterior $p_{1|t}^\rho$. Moreover, by \eqref{eq:population-velocity}, the conditional expectation of $v_{t|1}$ under $p_{1|t}^\rho(\cdot|x_t)$ equals the input velocity $v_t^\rho(x_t)$. Consequently, the original velocity condition becomes \eqref{eq:calibration-equation}, and the generalized velocity condition becomes \eqref{eq:generalized-optimality}. Consequently, under the assumptions of Theorem~\ref{thm:target-characterization} and for $\lambda_t\ne0$ on $(0,1)$,
\begin{equation*}
    \iter v^\rho=v^\rho
    \quad\Longleftrightarrow\quad
    \iter_{\lambda,\zeta}v^\rho=v^\rho
    \quad\Longleftrightarrow\quad\rho=\pi.
\end{equation*}

For a general trainable field $v$, the reduction above no longer applies, since $v$ need not coincide with the ideal flow model associated with its generated density. Theorem~\ref{thm:generalized-optimality} establishes a stronger characterization: under the stated assumptions and coefficient condition~\eqref{eq:generalized-uniqueness-condition}, the generalized velocity condition forces both $\rho^v=\pi$ and $v=v^\pi$, without presupposing that the candidate field is obtained through the CFM construction.

The preceding velocity theorems require a growth bound near the source and a weighted exponential moment of the generated density. The latter can be difficult to verify since $\rho^v$ is defined implicitly by the sampling dynamics. We provide another result (Proposition~\ref{prop:global-growth} in Appendix~\ref{app:velocity-proofs}) that gives sufficient conditions stated directly in terms of $v$ and the target score $\tau \nabla Q$: an integrable linear-growth bound over the full sampling interval, together with polynomial growth of the target score, implies both requirements without evaluating $\rho^v$.

We now turn the generalized velocity condition into a training objective. Its fixed-point form asks the current prediction to equal the conditional expectation defining $\iter_{\lambda,\zeta}v$. Rather than evaluate this expectation explicitly, we construct a sample-wise target from self-generated endpoints and independent source noise. Its conditional mean equals the operator output, so its difference from the current prediction provides a sampled correction for the fixed-point residual. The next subsection uses this construction in a stop-gradient objective.

\subsection{Learning from Self-Generated Samples}
\label{sec:practical-training}

We seek to enforce the generalized velocity condition through its fixed-point residual,
\begin{equation}
    (\iter_{\lambda,\zeta}v)_t(x_t)-v_t(x_t)=0.
    \label{eq:practical-fixed-point}
\end{equation}
The conditional-expectation structure of the generalized score-calibration operator provides a sample-wise training signal. For an admissible field $v$, combine the two expectations in \eqref{eq:generalized-score-calibration-operator} and place the current-prediction term inside the same expectation. The resulting integrand defines the target
\begin{equation}
    \target_{v,\lambda,\zeta}(t,x_t,x_1):=\zeta_tv_{t|1}(x_t|x_1)
    +(1-\lambda_t-\zeta_t)v_t(x_t)
    +\lambda_t\left(\frac{\dot\alpha_t}{\alpha_t}x_t
    +\frac{\kappa_t}{\alpha_t}\tau\nabla Q(x_1)\right).
    \label{eq:scf-target}
\end{equation}
By construction, whenever the expectation exists,
\begin{equation*}
    \E_{X_1\sim p_{1|t}^{\rho^v}(\cdot|x_t)}[\target_{v,\lambda,\zeta}(t,x_t,X_1)]
    =(\iter_{\lambda,\zeta}v)_t(x_t).
\end{equation*}

For the current network $v^\theta$, let $\rho^\theta:=\rho^{v^\theta}$. Generate $Y_1$ with the current flow and set $X_1:=Y_1\sim\rho^\theta$. Independently draw fresh source noise $X_0\sim p_0$ and $t\sim\U(0,1)$, then form $X_t=\alpha_tX_1+\beta_tX_0$. By Bayes' rule, the conditional density of $X_1$ given $(t,X_t=x_t)$ is $p_{1|t}^{\rho^\theta}(\cdot|x_t)$.

To seek a zero of the fixed-point residual in \eqref{eq:practical-fixed-point}, we use the sampled discrepancy as a correction signal, treating the target as fixed during differentiation. This gives the stop-gradient objective
\begin{equation*}
    \calL_{\lambda,\zeta}^{\mathrm{SCF}}(\theta)
    :=\E_{t\sim\U(0,1),\,X_1\sim\rho^{\sg(\theta)},\,X_t\sim p_{t|1}(\cdot|X_1)}
    \left[\norm[2]{v_t^\theta(X_t)-\sg(\target_{v^\theta,\lambda,\zeta}(t,X_t,X_1))}^2\right].
\end{equation*}
Here $\sg$ denotes the stop-gradient operator, which preserves the target value but blocks derivatives through all its terms.

It remains to specify the coefficients $\lambda_t$ and $\zeta_t$ in the SCF objective. A key consideration is the singular behavior of the score-based term in~\eqref{eq:scf-target} as $t\to0$. The coefficients must therefore cancel this singularity for numerical stability while satisfying~\eqref{eq:generalized-uniqueness-condition} to preserve fixed-point uniqueness. We provide concrete coefficient choices in Appendix~\ref{app:coefficient-proofs} and empirically demonstrate their effectiveness through ablation studies in Appendix~\ref{app:ablation}.

\subsection{Bounded Domains}
\label{sec:bounded-domains}

The preceding construction learns a sampler on $\R^d$. To accommodate bounded domains, such as the action spaces in many RL problems, we map an unconstrained latent variable into the desired domain and account for the change of variables in the target density. This allows us to retain the optimality conditions and training objective in latent space.

Consider a target $\pi(a)\propto e^{\tau Q(a)}$ on $(-1,1)^d$ and parameterize $a=\tanh x$ coordinatewise, with $x\in\R^d$. We use a tilde to denote the corresponding density in latent spaces. The change-of-variables formula gives
\begin{equation}
    \widetilde{\pi}(x)=\pi(\tanh x)\prod_{i=1}^d(1-\tanh^2x_i)
    \propto e^{\tau Q(\tanh x)}\prod_{i=1}^d(1-\tanh^2x_i).
    \label{eq:latent-density}
\end{equation}
Applying $\tanh$ ensures bounded outputs, but matching the intended density additionally requires the Jacobian factor above. The following proposition gives the resulting score and establishes the equivalence between the latent and bounded-domain sampling problems.

\begin{restatable}{proposition}{scfBoundedActions}
\label{prop:bounded-actions}
Let $Q\in C^1((-1,1)^d)$ and let $\pi$ be a probability density on $(-1,1)^d$ satisfying $\pi(a)\propto e^{\tau Q(a)}$. The latent target has score
\begin{equation}
    \nabla_x\log\widetilde{\pi}(x)=\tau\nabla_xQ(\tanh x)-2\tanh x.
    \label{eq:latent-score}
\end{equation}
Moreover, $Y_1\sim\widetilde{\pi}$ if and only if $\tanh Y_1\sim\pi$.
\end{restatable}

We therefore apply SCF in latent space with target $\widetilde{\pi}$, replacing $\tau\nabla Q$ in the optimality conditions and regression target with the latent score in \eqref{eq:latent-score}. Interpolation and training otherwise remain unchanged; during sampling, we apply $\tanh$ to the generated endpoint to obtain a sample in the original bounded domain.

\section{Applications to Generative Online Reinforcement Learning}
\label{sec:online-rl}

We apply SCF within an off-policy actor--critic framework. The overall procedure is summarized in Algorithm~\ref{alg:scf} in Appendix~\ref{app:experiment-details}. Without loss of generality, we take $\mathcal A=(-1,1)^d$ after affine rescaling of action bounds. The flow policy is parameterized by $v^\theta(\cdot,s)$, and $\pi^\theta(\cdot|s)$ denotes its induced action density. Sampling $a\sim\pi^\theta(\cdot|s)$ consists of solving
\begin{equation*}
    \frac{\dd Y_t}{\dd t}=v_t^\theta(Y_t,s),\qquad Y_0\sim p_0,
\end{equation*}
and returning $a=\tanh Y_1$.

\paragraph{Critic training.}
We maintain two critics $Q^{\omega_1}$ and $Q^{\omega_2}$, together with target networks parameterized by $\bar{\omega}_1$ and $\bar{\omega}_2$. Let $\mathcal D$ denote the replay buffer. Each critic, indexed by $i\in\{1,2\}$, minimizes
\begin{equation*}
    \calL_Q(\omega_i)
    :=\E_{(s,a,r,s')\sim\mathcal D,\,a'\sim\pi^{\theta}(\cdot|s')}
    \left[\left(Q^{\omega_i}(s,a)-\sg\left(r+\gamma\min_{j\in\{1,2\}}Q^{\bar{\omega}_j}(s',a')\right)\right)^2\right].
\end{equation*}

\paragraph{Actor training.}
Holding the state and critics fixed, SCF trains the flow policy toward the Boltzmann density $\pi^Q(\cdot|s)$ defined in \eqref{eq:boltzmann-policy}. We use $Q(s,a):=\min_{i\in\{1,2\}}Q^{\omega_i}(s,a)$ as the sampling potential.
Let $\widetilde{\pi}^{\theta}(\cdot|s)$ denote the density of the latent endpoint $Y_1$, so that $\pi^\theta(\cdot|s)=(\tanh)_\#\widetilde{\pi}^{\theta}(\cdot|s)$. The corresponding latent target $\widetilde{\pi}^{Q}(\cdot|s)$ is obtained from $\pi^Q(\cdot|s)$ through the change of variables in Section~\ref{sec:bounded-domains}, with score $\nabla_x\log\widetilde{\pi}^{Q}(x|s)=\tau\nabla_xQ(s,\tanh x)-2\tanh x$.
At each fixed state $s$, $\widetilde{\pi}^{Q}(\cdot|s)$ and $\widetilde{\pi}^{\theta}(\cdot|s)$ correspond, respectively, to the target density $\pi$ and generated density $\rho^\theta$ in the general formulation developed in the preceding sections.
The resulting actor loss is given by
\begin{equation*}
    \calL_{\lambda,\zeta}^{\mathrm{actor}}(\theta)
    :=\E_{\substack{s\sim\mathcal D,\,t\sim\U(0,1),\,X_1\sim\widetilde{\pi}^{\sg(\theta)}(\cdot|s),\,X_t\sim p_{t|1}(\cdot|X_1)}}
    \left[\norm[2]{v_t^\theta(X_t,s)-\sg(\target_{v^\theta,\lambda,\zeta}(t,X_t,X_1;s))}^2\right].
\end{equation*}

\section{Experiments}
\label{sec:experiments}

\subsection{Toy Example}
\label{sec:exp-toy}

We first validate SCF on a toy example, $\pi(x)\propto\exp(\tau Q(x))$, $x\in\R^2$, an equally weighted mixture of eight Gaussian components arranged on a circle. We compare with four recent generative samplers for unnormalized densities: Adjoint Sampling (AS)~\citep{havens2025adjoint}, Adjoint Schr\"odinger Bridge Sampler (ASBS)~\citep{liu2025adjoint}, Bridge Matching Sampler (BMS)~\citep{blessing2026bridge}, and Flow Sampling (FS)~\citep{havens2026flow}.
The results are summarized in Figure~\ref{fig:toy}. Experimental details are provided in Appendix~\ref{app:toy-example}. SCF achieves the lowest discrepancy from ground-truth samples while being the most efficient in terms of training and inference.

\begin{figure}[t]
    \centering
    \includegraphics[width=\textwidth]{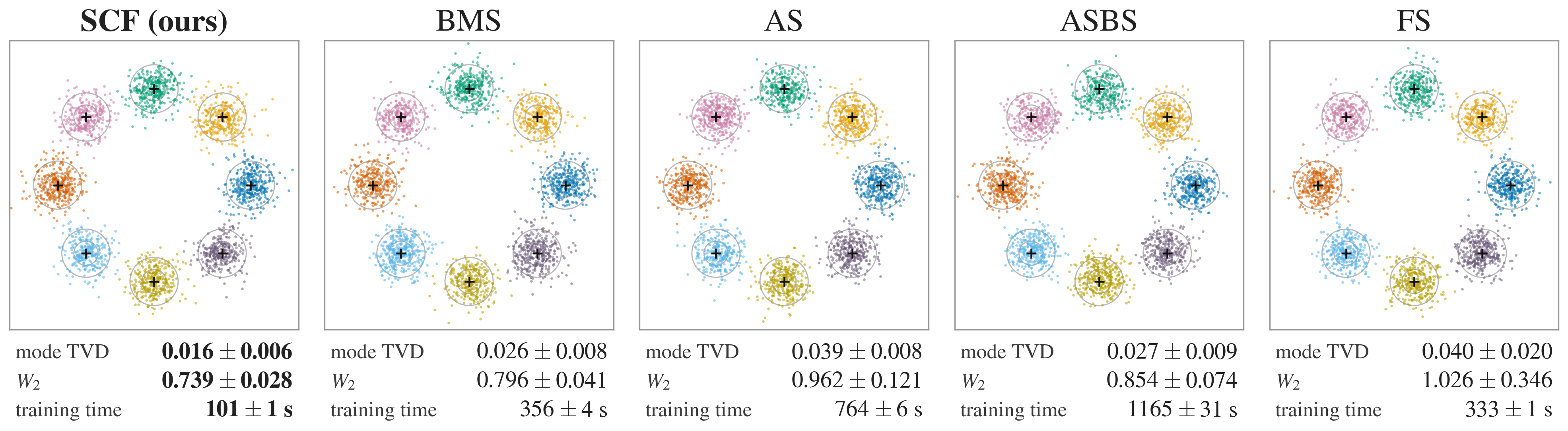}
    \caption{Visualization of samples from the five trained models on the eight-Gaussian target. SCF achieves the lowest discrepancy from ground-truth samples while being the cheapest to train and to sample from. The detailed descriptions are provided in Appendix~\ref{app:toy-example}.}
    \label{fig:toy}
\end{figure}

\subsection{RL Tasks}
\label{sec:exp-rl}

We evaluate SCF on two sets of tasks: widely used continuous-control tasks from the DeepMind Control Suite~\citep{tassa2018deepmind} and challenging humanoid-control tasks from HumanoidBench~\citep{sferrazza2024humanoidbench}. We compare SCF with soft actor-critic (SAC)~\citep{haarnoja2018soft}, a standard maximum-entropy RL algorithm with Gaussian policies, and eight strong generative-policy baselines spanning the four categories introduced in Section~\ref{sec:introduction}: DIME~\citep{celik2025dime}, Q-score matching (QSM)~\citep{psenka2024learning}, QFlex~\citep{wei2026scalable}, Q-weighted variational policy optimization (QVPO)~\citep{ding2024diffusion}, MaxEntDP~\citep{dong2025maximum}, soft diffusion actor-critic (SDAC)~\citep{ma2025efficient}, diffusion Q-sampling (DQS)~\citep{jain2025sampling}, and reverse flow matching (RFM)~\citep{li2026reverse}. Experimental details are provided in Appendix~\ref{app:rl-tasks}.

The training curves for HumanoidBench tasks are shown in Figure~\ref{fig:hb}. The training curves for DeepMind Control Suite tasks are deferred to Figure~\ref{fig:dmc} in Appendix~\ref{app:rl-tasks}. The results demonstrate that SCF consistently ranks among the top-performing methods across all tasks. Table~\ref{tab:update-time} in Appendix~\ref{app:rl-tasks} reports the wall-clock time per training iteration for each algorithm, with SCF requiring less time than all generative-policy baselines.

\begin{figure}[t]
    \centering
    \includegraphics[width=0.8\textwidth]{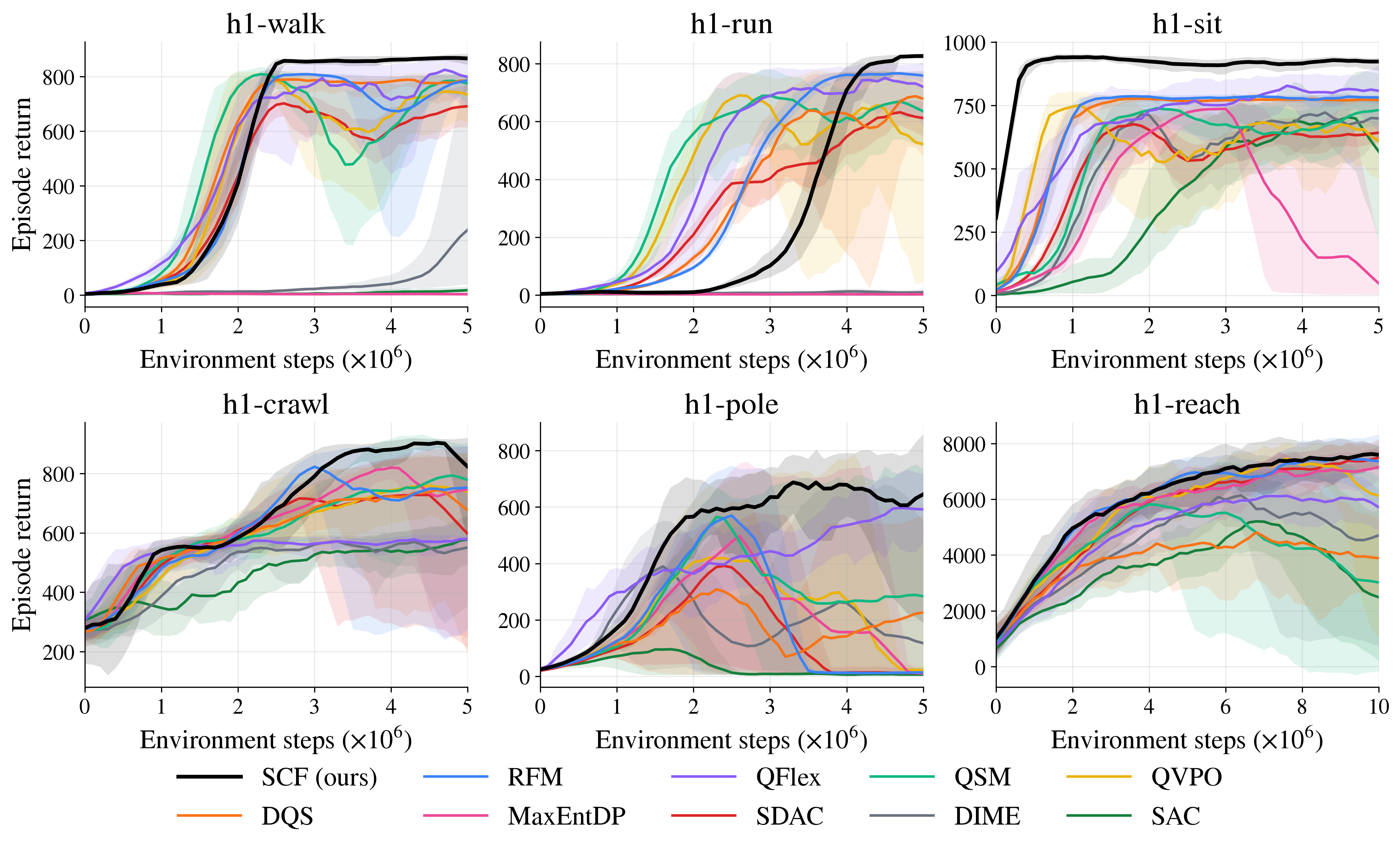}
    \caption{Training curves on six tasks from HumanoidBench. The solid lines correspond to the mean across five seeds, and the shaded regions represent minimum-to-maximum range.}
    \label{fig:hb}
\end{figure}

\section{Conclusion}
\label{sec:conclusion}

In this paper, we proposed Score-Calibrated Flow (SCF), a simple and efficient method for learning to sample from unnormalized densities with applications to generative online RL. By shifting from target posterior-mean estimation to self-consistency enforcement, we derived score-calibrated optimality conditions for both terminal densities and velocity fields. Their conditional-expectation structure yields a stop-gradient regression objective using self-generated samples. Experiments demonstrated strong performance and computational efficiency of SCF.

\bibliography{reference}
\bibliographystyle{iclr2027_conference}

\clearpage
\appendix

\section{Related Work}
\label{app:related_work}

\subsection{Generative Online Reinforcement Learning}
\label{app:related-generative-rl}

We organize existing methods by how they construct supervision signals for actor training.

\paragraph{Direct backpropagation.}
DACER~\citep{wang2024diffusion}, FlowRL~\citep{lv2025flow}, and DIME~\citep{celik2025dime} differentiate critic-based objectives through the sampling trajectory, incorporating additional approximations and regularization to make policy optimization more stable. Specifically, DACER estimates policy entropy by fitting Gaussian mixture models to sampled actions, FlowRL uses a flow-matching surrogate for Wasserstein regularization toward high-value replay-buffer behaviors, and DIME optimizes a tractable lower bound on the maximum-entropy objective. Despite these measures, backpropagation through the full sampling trajectory incurs substantial computational and memory costs and poses a nonnegligible risk of training instability.

\paragraph{Value-guided sampling.}
QSM~\citep{psenka2024learning} trains a diffusion score network using critic gradients at noisy actions, while QFlex~\citep{wei2026scalable} performs critic-gradient ascent from a learned source and fits a flow model to the resulting endpoints. Unlike our method, QSM and QFlex do not establish an exact characterization that uniquely identifies both the prescribed Boltzmann density and its velocity from CFM construction. Their guidance relies on repeated action-refinement steps, and evaluating the critic at noisy actions can expose it to inputs where its estimates are unreliable. We instead derive regression targets from score-calibrated optimality conditions and evaluate the critic gradient at clean actions. Each target requires only one endpoint-score evaluation, and the continuous-time formulation is independent of the number of sampling steps, without requiring an additional multi-step guidance procedure.

\paragraph{Weighted regression.}
QVPO~\citep{ding2024diffusion} trains a diffusion policy with Q-dependent weights and entropy regularization. Related approaches include reward-weighted flow matching~\citep{fan2025online} and diffusion policy mirror descent~\citep{ma2025efficient}. These methods use importance sampling to construct weighted regression objectives and therefore remain susceptible to high-variance training signals when the importance weights become concentrated. Generating and evaluating multiple candidate actions further increases computational cost. Moreover, their induced training densities depend on the proposal distribution and weighting rule and generally differ from the prescribed Boltzmann target $\pi^Q$, so even exact optimization may not recover the desired policy.

\paragraph{Boltzmann targeting.}
MaxEntDP~\citep{dong2025maximum} and SDAC~\citep{ma2025efficient} use Q-weighted noise signals to train noise-prediction and score networks, respectively. DQS~\citep{jain2025sampling} adapts iterated denoising energy matching~\citep{akhoundsadegh2024iterated} to train a score network using Q-weighted averages of critic gradients. RFM~\citep{li2026reverse} identifies these noise-expectation and gradient-expectation methods as instances of a shared posterior-mean estimation problem, and provides a principled way to combine both families through Stein control variates for variance reduction. These formulations target the Boltzmann density at the ideal population optimum, but their practical estimators rely on importance sampling, incurring computational overhead and estimator variance. SCF instead enforces score-calibrated self-consistency through regression on self-generated endpoints, without explicitly estimating the target posterior mean.

\subsection{Learning to Sample from Unnormalized Densities}
\label{app:related-unnormalized-sampling}

Learning generative samplers from unnormalized densities has been extensively studied, with representative approaches based on normalizing flows~\citep{noe2019boltzmann,papamakarios2021normalizing} and diffusion models~\citep{zhang2022path,vargas2023denoising,akhoundsadegh2024iterated}. We focus on recent matching-based methods that avoid both importance-weighted training targets and backpropagation through sampling trajectories.

Adjoint Sampling (AS)~\citep{havens2025adjoint} derives an endpoint-gradient regression objective from stochastic optimal control. Adjoint Schr\"odinger Bridge Sampler (ASBS)~\citep{liu2025adjoint} extends this construction beyond memoryless reference processes by alternating between learning a diffusion drift and a terminal corrector. Bridge Matching Sampler (BMS)~\citep{blessing2026bridge} derives diffusion-matching targets from a fixed-point iteration perspective. Flow Sampling (FS)~\citep{havens2026flow} constructs regression targets by combining the conditional velocity with an endpoint-score correction.

SCF grounds its training rule in optimality conditions that uniquely characterize the target density and ideal flow velocity under mild regularity requirements. BMS and FS adopt a fixed-point iteration perspective on training with self-generated samples, but do not establish that the desired sampler is the unique fixed point~\citep{blessing2026bridge,havens2026flow}. AS establishes fixed-point uniqueness only under the additional assumption of reciprocal-projection invariance~\citep{havens2025adjoint}. This is closer in spirit to our density optimality condition, where the flow is prescribed by the candidate density. Our velocity optimality condition goes further: it rules out spurious fixed points without presupposing a density--flow consistency. ASBS establishes convergence of its alternating updates under the assumption that each stage reaches its critical point, but does not establish uniqueness of the joint critical points of its coupled drift--corrector matching equations~\citep{liu2025adjoint}. SCF instead derives necessary and sufficient conditions directly for the terminal density and trainable velocity. Their unique solutions are the target density $\pi$ and the ideal velocity $v^\pi$ that CFM would learn from target samples. Moreover, the four methods above learn the drifts of stochastic differential equations (SDEs). Their sampling procedures rely on stochastic simulation, which may require a large number of function evaluations (NFEs), increasing computational costs during both training and inference. FS illustrates a structural limitation of this formulation: its target-score correction is proportional to the squared diffusion coefficient, so setting this coefficient to zero eliminates explicit target-score supervision~\citep{havens2026flow}. The remaining objective reduces to CFM on self-generated endpoints, which reproduces the current density without identifying the prescribed target. Thus, simply removing the diffusion noise does not yield a training objective that ensures the learned velocity field generates the prescribed target density through ODE sampling. In contrast, SCF derives its velocity regression objective from score-calibrated optimality conditions, retaining target-score supervision without relying on stochastic sampling dynamics.

\section{Proofs of Theoretical Results}
\label{app:proofs}

We state the mild regularity conditions in Appendix~\ref{app:regularity} and provide the proofs in the subsequent subsections. For convenience, we restate each proposition or theorem from the main text before its proof.

\subsection{Regularity Conditions}
\label{app:regularity}

The schedules satisfy the assumptions in Section~\ref{sec:cfm-background} and $\kappa_t>0$ from \eqref{eq:kappa}. Identities with divisions by $\alpha_t$ or $\beta_t$ are stated for $t\in(0,1)$. Velocity equalities are understood on $(0,1)\times\R^d$.

\paragraph{Score identities.}
For every density $\rho$ to which Proposition~\ref{prop:score-velocity} is applied, assume positivity, normalization, and $C^1$ regularity. For every $t\in(0,1)$ and $x_t\in\R^d$, assume
\begin{equation}
    \E_{X_1\sim p_{1|t}^\rho(\cdot|x_t)}
    [\norm[2]{X_1}]<\infty,\qquad
    \E_{X_1\sim p_{1|t}^\rho(\cdot|x_t)}
    [\norm[2]{\nabla\log\rho(X_1)}]<\infty.
    \label{eq:posterior-moment-condition}
\end{equation}
For every density $\rho$ to which Theorem~\ref{thm:target-characterization} is applied, additionally assume, for every $t\in(0,1)$ and $x_t\in\R^d$, that
\begin{equation}
    \E_{X_1\sim p_{1|t}^\rho(\cdot|x_t)}
    [\norm[2]{\nabla Q(X_1)}]<\infty.
    \label{eq:nabla-Q-integrability-on-rho-posterior}
\end{equation}

\paragraph{Velocity fields.}
We call a velocity field $v$ \emph{admissible} if its source-initialized ODE
\begin{equation*}
    \frac{\dd}{\dd t}\flow_{0\to t}^v(x_0)=v_t(\flow_{0\to t}^v(x_0)),
    \qquad \flow_{0\to0}^v(x_0)=x_0,
\end{equation*}
has a unique solution on $[0,1]$ with measurable flow maps, its terminal pushforward has a positive $C^1$ density $\rho^v$, and the conditional moments defining the score-calibration operators are finite.

For a population CFM field $v^\rho$, assume that its ODE transports $p_0$ along $p_t^\rho$ to $\rho$. These requirements also apply to $v^\pi$ when asserting existence of the target flow.

For every admissible field $v$ to which Theorem~\ref{thm:velocity-optimality} is applied, assume that, for some $t_0\in(0,1)$, $\delta>0$, and nonnegative measurable functions $a,b$ on $(0,t_0)$,
\begin{equation}
    \norm[2]{v_t(x_t)}\le a_t+b_t\norm[2]{x_t},
    \qquad \frac{\alpha_t}{\beta_t^2}a_t\in L^1(0,t_0),\qquad b\in L^1(0,t_0),
    \label{eq:source-growth}
\end{equation}
for almost every $t\in(0,t_0)$ and $x_t$ with $\norm[2]{x_t}\le \frac{\beta_t^2}{\alpha_t}\delta$. Additionally assume that, for some $\varepsilon>0$,
\begin{equation}
    \int_{\R^d}e^{\varepsilon\norm[2]{x}}\rho^v(x)
    \bigl(1+\norm[2]{\tau\nabla Q(x)}\bigr)\,\dd x<\infty.
    \label{eq:weighted-exponential-moment}
\end{equation}

\subsection{Density Optimality Conditions}
\label{app:target-proofs}

\begin{proposition}\label{prop:conditional_score_zero_mean}
    Let $\rho$ be a positive, normalized $C^1$ density satisfying the condition \eqref{eq:posterior-moment-condition}. For every $t\in(0,1)$ and $x_t\in\R^d$,
    \[
    \int_{\R^d}\nabla_{x_1}p_{1|t}^\rho(x_1|x_t)\,\dd x_1=0.
    \]
\end{proposition}

\begin{proof}
    Fix $t\in(0,1)$ and $x_t\in\R^d$. For every $R>0$, denote $B_R:=\{x\in\R^d:\norm[2]{x}\leq R\}$ by the closed ball of radius $R$. Choose a smooth cutoff function $\varphi_1\in C^\infty(\R^d)$ such that
\[
0\leq\varphi_1\leq 1,\qquad \varphi_1\equiv 1\text{ on }B_1,\qquad \supp\,\varphi_1\subset B_2,\qquad C:=\sup_{x\in\R^d}\norm[2]{\nabla\varphi_1(x)}<\infty,
\]
where $\supp$ denotes the support set of a function. For every $R>0$, define $\varphi_R(x):=\varphi_1\left(\frac{x}{R}\right)$. Then
\[
0\leq\varphi_R\leq 1,\qquad \varphi_R\equiv 1\text{ on }B_R,\qquad \supp\,\varphi_R\subset B_{2R},\qquad \sup_{x\in\R^d}\norm[2]{\nabla\varphi_R(x)}\leq \frac{C}{R}.
\]
Since $\varphi_R$ has compact support, integrating by parts yields
\[
    \int_{\R^d}\varphi_R(x_1)\nabla_{x_1}p_{1|t}^\rho(x_1|x_t)\,\dd x_1=-\int_{\R^d}p_{1|t}^\rho(x_1|x_t)\nabla\varphi_R(x_1)\,\dd x_1.
\]
For every $R>0$, the integrand on the left-hand side satisfies
\begin{align*}
    &\norm[2]{\varphi_R(x_1)\nabla_{x_1}p_{1|t}^\rho(x_1|x_t)}\leq \norm[2]{\nabla_{x_1} p_{1|t}^\rho(x_1|x_t)}
    \quad\text{for every }x_1.
\end{align*}
By Bayes' rule and differentiation of the Gaussian likelihood,
\[
    \nabla_{x_1}p_{1|t}^\rho(x_1|x_t)
    =
    p_{1|t}^\rho(x_1|x_t)
    \left[
        \nabla\log\rho(x_1)
        +\frac{\alpha_t}{\beta_t^2}(x_t-\alpha_t x_1)
    \right].
\]
Consequently, \eqref{eq:posterior-moment-condition} gives
\begin{align*}
\int_{\R^d}\norm[2]{\nabla_{x_1}p_{1|t}^\rho(x_1|x_t)}\,\dd x_1
=&\E_{X_1\sim p_{1|t}^\rho(\cdot|x_t)}\left[\norm[2]{\nabla\log\rho(X_1)+\frac{\alpha_t}{\beta_t^2}(x_t-\alpha_tX_1)}\right]\\
\leq&\E_{X_1\sim p_{1|t}^\rho(\cdot|x_t)}\left[\norm[2]{\nabla\log\rho(X_1)}+
\frac{\alpha_t^2}{\beta_t^2}\norm[2]{X_1}\right]+\frac{\alpha_t}{\beta_t^2}\norm[2]{x_t}\\
<&\infty.
\end{align*}
Since
\[
    \norm[2]{\int_{\R^d}p_{1|t}^\rho(x_1|x_t)\nabla\varphi_R(x_1)\,\dd x_1}
    \leq \frac{C}{R}\abs{\int_{\R^d}p_{1|t}^\rho(x_1|x_t)\,\dd x_1}=\frac{C}{R}\to 0\quad\text{as }R\to\infty,
\]
Since $\varphi_R(x_1)\to1$ for every $x_1$ and the preceding integrable bound is independent of $R$, Lebesgue's dominated convergence theorem \citep[Theorem~1.34]{rudin1987real}, applied componentwise, gives
\begin{align*}
    \int_{\R^d}\nabla_{x_1}p_{1|t}^\rho(x_1|x_t)\,\dd x_1
    =&\lim_{R\to\infty}\int_{\R^d}\varphi_R(x_1)\nabla_{x_1}p_{1|t}^\rho(x_1|x_t)\,\dd x_1\\
    =&-\lim_{R\to\infty}\int_{\R^d}p_{1|t}^\rho(x_1|x_t)\nabla\varphi_R(x_1)\,\dd x_1\\
    =&0.
\end{align*}
This proves the proposition.
\end{proof}

\restateproposition{\scfScoreVelocity}

\begin{proof}
Fix $t\in(0,1)$ and $x_t\in\R^d$. By Bayes' rule,
\begin{equation*}
    \log p_{1|t}^\rho(x_1|x_t)
    =\log\rho(x_1)+\log p_{t|1}(x_t|x_1)-\log p_t^\rho(x_t).
\end{equation*}
The last term is independent of $x_1$. Differentiating the Gaussian conditional density therefore gives
\begin{equation}
    \nabla_{x_1}\log p_{1|t}^\rho(x_1|x_t)
    =\nabla\log\rho(x_1)+\frac{\alpha_t}{\beta_t^2}(x_t-\alpha_tx_1).
    \label{eq:proof-posterior-score}
\end{equation}
Proposition~\ref{prop:conditional_score_zero_mean} implies
\[
    \E_{X_1\sim p_{1|t}^\rho(\cdot|x_t)}\left[\nabla_{x_1}\log p_{1|t}^\rho(X_1|x_t)\right]
    =\int_{\R^d}\nabla_{x_1}p_{1|t}^\rho(x_1|x_t)\,\dd x_1
    =0.
\]
Consequently, taking the expectation of \eqref{eq:proof-posterior-score} under the same posterior $p_{1|t}^\rho(\cdot|x_t)$ gives
\begin{equation*}
    \E_{X_1\sim p_{1|t}^\rho(\cdot|x_t)}[\nabla\log\rho(X_1)]
    =-\frac{\alpha_t}{\beta_t^2}x_t
    +\frac{\alpha_t^2}{\beta_t^2}\E_{X_1\sim p_{1|t}^\rho(\cdot|x_t)}[X_1].
\end{equation*}
Then, the conditional-velocity formula and the definition of $\kappa_t$ imply
\begin{align*}
    v_t^\rho(x_t)
    =&\frac{\dot\beta_t}{\beta_t}x_t
    +\frac{\alpha_t\kappa_t}{\beta_t^2}\E_{X_1\sim p_{1|t}^\rho(\cdot|x_t)}[X_1]\\
    =&\left(\frac{\dot\beta_t}{\beta_t}+\frac{\kappa_t}{\beta_t^2}\right)x_t
    +\frac{\kappa_t}{\alpha_t}\E_{X_1\sim p_{1|t}^\rho(\cdot|x_t)}[\nabla\log\rho(X_1)].\notag
\end{align*}
Substituting $\frac{\dot\beta_t}{\beta_t}+\frac{\kappa_t}{\beta_t^2}=\frac{\dot\alpha_t}{\alpha_t}$ yields \eqref{eq:score-velocity}.
\end{proof}

\begin{proposition}
\label{prop:exponential-average-uniqueness}
Let $h\in L^1(\R^d;\R^d)$ satisfy
\[
    \int_{\R^d}e^{\omega^\top x}h(x)\,\dd x=0,
    \qquad
    \int_{\R^d}e^{\omega^\top x}\norm[2]{h(x)}\,\dd x<\infty
\]
for every $\omega\in\R^d$. Then $h=0$ almost everywhere.
\end{proposition}

\begin{proof}
By the injectivity of the Fourier transform on $L^1(\R^d)$ \citep[Corollary~8.27]{folland1999real}, applied componentwise, it suffices to show that
$$
    \Fourier[h](\omega)
    :=
    \int_{\R^d}e^{-\mathrm{i}\omega^\top x}h(x)\,\dd x
$$
is zero for every $\omega\in\R^d$.

Fix $\omega\in\R^d$ and define $F_\omega:\C\to\C^d$ by
\[
    F_\omega(z):=\int_{\R^d}e^{z\omega^\top x}h(x)\,\dd x.
\]
This integral is well-defined for every $z\in\C$, since
\[
    \int_{\R^d}\norm[2]{e^{z\omega^\top x}h(x)}\,\dd x
    =
    \int_{\R^d}e^{(\Re z)\omega^\top x}\norm[2]{h(x)}\,\dd x
    <\infty.
\]
By assumption, $F_\omega(z)=0$ for every $z\in\R$, while
\[
    \Fourier[h](\omega)
    =
    F_\omega(-\mathrm{i}).
\]
It therefore remains to verify that $F_\omega$ is entire componentwise: the identity theorem will then extend its vanishing on the real axis to all of $\C$, giving $F_\omega(-\mathrm{i})=0$.

To establish this, observe that for every $R>0$,
\begin{align*}
    \int_{\R^d}\sup_{\abs{z}\le R}\norm[2]{e^{z\omega^\top x}h(x)}\,\dd x
    &\le \int_{\R^d}e^{R\abs{\omega^\top x}}\norm[2]{h(x)}\,\dd x\\
    &\le \int_{\R^d}\left(e^{R\omega^\top x}+e^{-R\omega^\top x}\right)\norm[2]{h(x)}\,\dd x
    <\infty.
\end{align*}
The integrand is measurable in $x$ and entire in $z$ for each $x$. Hence the locally dominated form of the holomorphic parameter-integral theorem, a consequence of \citet[Chapter~10, Exercise~16]{rudin1987real}, implies that $F_\omega$ is holomorphic componentwise on $\abs{z}<R$. Since $R>0$ is arbitrary, $F_\omega$ is entire componentwise.

Then, applying the identity theorem componentwise gives $F_\omega\equiv0$ on $\C$. Thus, $\Fourier[h](\omega)=F_\omega(-\mathrm{i})=0$. Since $\omega\in\R^d$ was arbitrary, Fourier injectivity yields $h=0$ almost everywhere.
\end{proof}

\restatetheorem{\scfTargetCharacterization}

\begin{proof}
    Suppose $\rho=\pi$, then $\nabla\log\rho=\tau\nabla Q$; hence, \eqref{eq:nabla-Q-integrability-on-rho-posterior} and Proposition~\ref{prop:score-velocity} give \eqref{eq:calibration-equation} for every $t\in(0,1)$. The implication from every interior time to some interior time is immediate.

    Suppose \eqref{eq:calibration-equation} holds for some $t\in(0,1)$. Since $\frac{\kappa_t}{\alpha_t}>0$, Proposition~\ref{prop:score-velocity} gives
\begin{equation*}
    \E_{X_1\sim p_{1|t}^\rho(\cdot|x_t)}[\nabla\log\rho(X_1)-\tau\nabla Q(X_1)]=0.
\end{equation*}
Since $p_t^\rho(x_t)>0$, expanding the posterior and multiplying by this marginal density yields
\begin{equation}
    \int_{\R^d}p_{t|1}(x_t|x_1)\rho(x_1)(\nabla\log\rho(x_1)-\tau\nabla Q(x_1))\,\dd x_1=0\quad\text{for every }x_t\in\R^d.
    \label{eq:proof-gaussian-integral}
\end{equation}

We now show that \eqref{eq:proof-gaussian-integral} forces the candidate and target scores to coincide. Define
\[
h_t(x_1):=\rho(x_1)(\nabla\log\rho(x_1)-\tau\nabla Q(x_1))\exp\left(-\frac{\alpha_t^2}{2\beta_t^2}\norm[2]{x_1}^2\right).
\]
Since 
$$p_{t|1}(x_t|x_1)=\Normal(x_t;\alpha_tx_1,\beta_t^2I)\propto\exp\left(-\frac{1}{2\beta_t^2}\norm[2]{x_t-\alpha_tx_1}^2\right),$$
\eqref{eq:proof-gaussian-integral} becomes
\[
    \int_{\R^d}h_t(x_1)\exp\left(\frac{\alpha_t}{\beta_t^2}x_t^\top x_1\right)\,\dd x_1=0\quad\text{for every }x_t\in\R^d.
\]
Then, \eqref{eq:posterior-moment-condition} and \eqref{eq:nabla-Q-integrability-on-rho-posterior} imply that
\[
    \int_{\R^d}\norm[2]{h_t(x_1)}\exp\left(\frac{\alpha_t}{\beta_t^2}x_t^\top x_1\right)\,\dd x_1<\infty\quad\text{for every }x_t\in\R^d.
\] 
Taking $x_t=0$ in the preceding integrability bound gives $h_t\in L^1(\R^d;\R^d)$. Moreover, $\frac{\alpha_t}{\beta_t^2}x_t$ ranges over $\R^d$ as $x_t$ varies. Therefore, Proposition~\ref{prop:exponential-average-uniqueness} gives
\[
h_t(x_1)=0\quad\text{for a.e. }x_1.
\]
By the positivity of $\rho$ and the continuity of $\nabla\log\rho$, $\nabla Q$, this implies
\[
\nabla\log\rho(x_1)=\tau\nabla Q(x_1)=\nabla\log\pi(x_1)\quad\text{for every }x_1\in\R^d.
\]
The normalization of $\rho$ and $\pi$ then gives $\rho=\pi$.
\end{proof}

\restateproposition{\scfAffineCalibration}

\begin{proof}
    By \eqref{eq:population-velocity} and $\lambda_t\neq0$, \eqref{eq:generalized-optimality} is equivalent to \eqref{eq:calibration-equation}. Therefore, this proposition is a straightforward conclusion of Theorem~\ref{thm:target-characterization}.
\end{proof}

\subsection{Velocity Optimality Conditions}
\label{app:velocity-proofs}

\begin{proposition}
\label{prop:exponential-tilt-identification}
Let $\pi$ be the target density in \eqref{eq:target-density}, and let $\rho$ be a positive, normalized $C^1$ density satisfying
\begin{equation}\label{eq:exponential-zero-assumption}
    \int_{\R^d}e^{\omega^\top x}\rho(x)\left(\tau\nabla Q(x)+\omega\right)\,\dd x=0
\end{equation}
for every $\omega$ in a neighborhood of $0\in\R^d$, and
\begin{equation}\label{eq:exponential-moment-assumption}
    \int_{\R^d}e^{\varepsilon\norm[2]{x}}\rho(x)\left(\norm[2]{\tau\nabla Q(x)}+1\right)\,\dd x<\infty
\end{equation}
for some $\varepsilon>0$. Then $\rho=\pi$.
\end{proposition}

\begin{proof}
For every $\omega\in\R^d$, consider the Fourier transforms
\[
    \Fourier[\rho\tau\nabla Q](\omega)
    :=
    \int_{\R^d}e^{-\mathrm{i}\omega^\top x}\rho(x)\tau\nabla Q(x)\,\dd x,\qquad
    \Fourier[\rho](\omega)
    :=
    \int_{\R^d}e^{-\mathrm{i}\omega^\top x}\rho(x)\,\dd x.
\]
Both integrals converge absolutely by \eqref{eq:exponential-moment-assumption}. To prove $\rho=\pi$, it suffices to establish the score identity
\begin{equation}\label{eq:score-identity-exponential-tilt-identification}
    \rho\tau\nabla Q=\nabla\rho.
\end{equation}
The formal Fourier differentiation identity $\Fourier[\nabla\rho](\omega)=\mathrm{i}\omega\Fourier[\rho](\omega)$ suggests a Fourier-based argument. However, we have not assumed that $\nabla\rho$ is integrable, so its Fourier transform is not yet justified as an absolutely convergent integral. We therefore proceed in two steps: first, we establish
\begin{equation}\label{eq:Fourier-identity-exponential-tilt-identification}
    \Fourier[\rho\tau\nabla Q](\omega)=\mathrm{i}\omega\Fourier[\rho](\omega)
    \quad\text{for every }\omega\in\R^d;
\end{equation}
then, we recover \eqref{eq:score-identity-exponential-tilt-identification} from this Fourier identity using compactly supported smooth test functions.

Fix a nonzero $\omega\in\R^d$ and define $F_\omega$ on the strip $\abs{\Re z}<\frac{\varepsilon}{\norm[2]{\omega}}$ by
\begin{equation*}
    F_\omega(z):=\int_{\R^d}e^{z\omega^\top x}\rho(x)\left(\tau\nabla Q(x)+z\omega\right)\,\dd x.
\end{equation*}
This integral is well-defined for every $z$ in the strip by \eqref{eq:exponential-moment-assumption}. Applying \eqref{eq:exponential-zero-assumption} at $s\omega$ gives $F_\omega(s)=0$ for every real $s$ sufficiently close to $0$, while
\[
    F_\omega(-\mathrm{i})=\Fourier[\rho\tau\nabla Q](\omega)-\mathrm{i}\omega\Fourier[\rho](\omega).
\]
To obtain the desired Fourier identity \eqref{eq:Fourier-identity-exponential-tilt-identification}, it therefore remains to verify that $F_\omega$ is holomorphic componentwise on the strip.

To establish this, fix a compact subset $K$ of the strip and choose $R>0$ such that $\abs{z}\le R$ for every $z\in K$. Then \eqref{eq:exponential-moment-assumption} implies
\begin{align*}
    \int_{\R^d}\sup_{z\in K}\norm[2]{e^{z\omega^\top x}\rho(x)\left(\tau\nabla Q(x)+z\omega\right)}\dd x\le \int_{\R^d}e^{\varepsilon\norm[2]{x}}\rho(x)\left(\norm[2]{\tau\nabla Q(x)}+R\norm[2]{\omega}\right)\dd x
    <\infty.
\end{align*}
The integrand is measurable in $x$ and entire in $z$ for each $x$. Therefore, the locally dominated form of the holomorphic parameter-integral theorem, a consequence of \citet[Chapter~10, Exercise~16]{rudin1987real}, implies that $F_\omega$ is holomorphic componentwise on the strip.

The identity theorem \citep[Corollary of Theorem~10.18]{rudin1987real}, applied componentwise, now gives $F_\omega\equiv0$ throughout the connected strip. In particular, $F_\omega(-\mathrm{i})=0$ gives
\begin{equation*}
    \Fourier[\rho\tau\nabla Q](\omega)=\mathrm{i}\omega\Fourier[\rho](\omega).
\end{equation*}
This holds for every $\omega\neq0$. For $\omega=0$, the same identity follows from \eqref{eq:exponential-zero-assumption} at the origin. Thus \eqref{eq:Fourier-identity-exponential-tilt-identification} holds for every $\omega\in\R^d$.

It remains to recover the score identity \eqref{eq:score-identity-exponential-tilt-identification}. Fix any smooth function $\varphi:\R^d\to\R$ with compact support, and denote
\[
    \Fourier[\varphi](\omega)
    :=
    \int_{\R^d}e^{-\mathrm{i}\omega^\top x}\varphi(x)\,\dd x.
\]
Since $\varphi$ is smooth with compact support, we have \cite[Corollary~8.23]{folland1999real}
\[
    \int_{\R^d}\left(1+\norm[2]{\omega}\right)\abs{\Fourier[\varphi](\omega)}\,\dd\omega<\infty.
\]
Together with the integrability of $\rho$ and $\rho\tau\nabla Q$, this justifies the applications of Fubini's theorem below. Fourier inversion \citep[Theorem~8.26]{folland1999real} and Fubini's theorem, followed by \eqref{eq:Fourier-identity-exponential-tilt-identification} at $-\omega$, give
\begin{align*}
    \int_{\R^d}\varphi(x)\rho(x)\tau\nabla Q(x)\,\dd x
    &=\frac{1}{(2\pi)^d}\int_{\R^d}\Fourier[\varphi](\omega)\Fourier[\rho\tau\nabla Q](-\omega)\,\dd\omega\\
    &=-\frac{1}{(2\pi)^d}\int_{\R^d}\mathrm{i}\omega\Fourier[\varphi](\omega)\Fourier[\rho](-\omega)\,\dd\omega.
\end{align*}
The Fourier differentiation identity \citep[Theorem~8.22]{folland1999real} applies to the test function $\varphi$ and gives $\Fourier[\nabla\varphi](\omega)=\mathrm{i}\omega\Fourier[\varphi](\omega)$. Hence,
\begin{align*}
    -\frac{1}{(2\pi)^d}\int_{\R^d}\mathrm{i}\omega\Fourier[\varphi](\omega)\Fourier[\rho](-\omega)\,\dd\omega
    &=-\frac{1}{(2\pi)^d}\int_{\R^d}\Fourier[\nabla\varphi](\omega)\Fourier[\rho](-\omega)\,\dd\omega\\
    &=-\int_{\R^d}\rho(x)\nabla\varphi(x)\,\dd x\\
    &=\int_{\R^d}\varphi(x)\nabla\rho(x)\,\dd x.
\end{align*}
Here the second equality follows from Fourier inversion applied componentwise to $\nabla\varphi$ and Fubini's theorem. The last equality is integration by parts, justified by $\rho\in C^1$ and the compact support of $\varphi$; no global integrability of $\nabla\rho$ is required.

Therefore, for every smooth function $\varphi:\R^d\to\R$ with compact support,
\[
    \int_{\R^d}\varphi(x)\rho(x)\tau\nabla Q(x)\,\dd x
    =
    \int_{\R^d}\varphi(x)\nabla\rho(x)\,\dd x.
\]
This implies \eqref{eq:score-identity-exponential-tilt-identification} almost everywhere. The continuity of $\nabla\rho$ and $\rho\tau\nabla Q$ then gives \eqref{eq:score-identity-exponential-tilt-identification} pointwise on $\R^d$. By positivity of $\rho$,
\[
    \nabla\log\rho=\tau\nabla Q=\nabla\log\pi.
\]
The normalization of $\rho$ and $\pi$ therefore yields $\rho=\pi$.
\end{proof}

\begin{proposition}
\label{prop:tilted-posterior-limit}
Let $\rho$ be a positive, normalized $C^1$ density on $\R^d$. Assume that \eqref{eq:weighted-exponential-moment} holds with $\rho^v$ replaced by $\rho$ for some $\varepsilon>0$. Let $\varphi\in C(\R^d;\R^d)$ satisfy
\[
    \int_{\R^d}e^{\xi^\top x}\rho(x)\norm[2]{\varphi(x)}\,\dd x<\infty
    \qquad\text{for every }\xi\in\R^d\text{ with }\norm[2]{\xi}<\varepsilon.
\]
Then, for every fixed $\xi\in\R^d$ with $\norm[2]{\xi}<\varepsilon$,
\begin{equation}
    \lim_{t\downarrow0}
    \E_{X_1\sim p_{1|t}^{\rho}\left(\cdot\middle|c_t\xi\right)}
    [\varphi(X_1)]
    =
    \frac{\int_{\R^d}e^{\xi^\top x}\rho(x)\varphi(x)\,\dd x}
    {\int_{\R^d}e^{\xi^\top x}\rho(x)\,\dd x},
    \label{eq:tilted-score-limit-general}
\end{equation}
where $c_t:=\frac{\beta_t^2}{\alpha_t}$. The expectations and integrals in \eqref{eq:tilted-score-limit-general} are absolutely convergent, and the limiting denominator is strictly positive.

In particular, \eqref{eq:weighted-exponential-moment} guarantees the required integrability for both $\varphi(x)=x$ and $\varphi(x)=\tau\nabla Q(x)$. Consequently, for every fixed $\xi\in\R^d$ with $\norm[2]{\xi}<\varepsilon$,
\begin{subequations}\label{eq:tilted-score-limit-X1-and-nablaQ}
    \begin{align}
        &\lim_{t\downarrow0}
        \E_{X_1\sim p_{1|t}^{\rho}\left(\cdot\middle|c_t\xi\right)}
        [X_1]
        =
        \frac{\int_{\R^d}e^{\xi^\top x}\rho(x)x\,\dd x}
        {\int_{\R^d}e^{\xi^\top x}\rho(x)\,\dd x},
        \label{eq:tilted-score-limit-X1}\\
        &\lim_{t\downarrow0}
        \E_{X_1\sim p_{1|t}^{\rho}\left(\cdot\middle|c_t\xi\right)}
        [\tau\nabla Q(X_1)]
        =
        \frac{\int_{\R^d}e^{\xi^\top x}\rho(x)\tau\nabla Q(x)\,\dd x}
        {\int_{\R^d}e^{\xi^\top x}\rho(x)\,\dd x}.
        \label{eq:tilted-score-limit-nablaQ}
    \end{align}
\end{subequations}
Both limits are finite. In particular, for each such $\xi$, both posterior means are bounded for all sufficiently small $t$.
\end{proposition}

\begin{proof}
Fix $\xi\in\R^d$ with $\norm[2]{\xi}<\varepsilon$. Expanding the Gaussian likelihood at $x_t=c_t\xi$ gives
\[
    -\frac{1}{2\beta_t^2}
    \norm[2]{c_t\xi-\alpha_t x_1}^2
    =
    -\frac{\beta_t^2}{2\alpha_t^2}\norm[2]{\xi}^2
    +\xi^\top x_1
    -\frac{\alpha_t^2}{2\beta_t^2}\norm[2]{x_1}^2.
\]
The first term is independent of $x_1$ and cancels in posterior normalization, yielding
\begin{equation}
    p_{1|t}^{\rho}\left(x_1\middle|c_t\xi\right)
    =
    \frac{\exp\left(\xi^\top x_1-\frac{\alpha_t^2}{2\beta_t^2}\norm[2]{x_1}^2\right)\rho(x_1)}
    {\int_{\R^d}\exp\left(\xi^\top x-\frac{\alpha_t^2}{2\beta_t^2}\norm[2]{x}^2\right)\rho(x)\,\dd x}.
    \label{eq:tilted-posterior}
\end{equation}

We first establish
\begin{subequations}\label{eq:limitequation-tilted-posterior-limit}
\begin{align}
    &\lim_{t\downarrow0}\int_{\R^d}\exp\left(\xi^\top x-\frac{\alpha_t^2}{2\beta_t^2}\norm[2]{x}^2\right)\rho(x)\,\dd x
    =\int_{\R^d}e^{\xi^\top x}\rho(x)\,\dd x,\\
    &\lim_{t\downarrow0}\int_{\R^d}\exp\left(\xi^\top x-\frac{\alpha_t^2}{2\beta_t^2}\norm[2]{x}^2\right)\rho(x)\varphi(x)\,\dd x
    =\int_{\R^d}e^{\xi^\top x}\rho(x)\varphi(x)\,\dd x.
\end{align}
\end{subequations}

Both integrands are measurable since $\rho$ and $\varphi$ are continuous. Since $\alpha,\beta$ are continuous with $\alpha_0=0$ and $\beta_0=1$, we have $\frac{\alpha_t}{\beta_t}\to0$ as $t\downarrow0$. Thus, for every $x\in\R^d$,
\begin{align*}
    &\lim_{t\downarrow0}\exp\left(\xi^\top x-\frac{\alpha_t^2}{2\beta_t^2}\norm[2]{x}^2\right)\rho(x)
    =e^{\xi^\top x}\rho(x),\\
    &\lim_{t\downarrow0}\exp\left(\xi^\top x-\frac{\alpha_t^2}{2\beta_t^2}\norm[2]{x}^2\right)\rho(x)\varphi(x)
    =e^{\xi^\top x}\rho(x)\varphi(x).
\end{align*}
For every $t\in(0,1)$ and $x\in\R^d$,
\begin{align*}
    &0<
    \exp\left(\xi^\top x-\frac{\alpha_t^2}{2\beta_t^2}\norm[2]{x}^2\right)\rho(x)
    \le
    e^{\xi^\top x}\rho(x)
    \le
    e^{\varepsilon\norm[2]{x}}\rho(x),\\
    &\norm[2]{
        \rho(x)\exp\left(\xi^\top x-\frac{\alpha_t^2}{2\beta_t^2}\norm[2]{x}^2\right)\varphi(x)
    }
    \le
    e^{\xi^\top x}\rho(x)\norm[2]{\varphi(x)}.
\end{align*}
The right-hand sides are independent of $t$. The scalar bound is integrable by \eqref{eq:weighted-exponential-moment}, and the vector bound is integrable by the assumption on $\varphi$. Therefore, Lebesgue's dominated convergence theorem \citep[Theorem~1.34]{rudin1987real}, applied componentwise to the vector-valued integral, implies \eqref{eq:limitequation-tilted-posterior-limit}.

The same bounds ensure absolute convergence of both integrals for every $t\in(0,1)$ and of their limits. Since $\rho>0$, the denominator in \eqref{eq:tilted-posterior} and its limit are strictly positive. Taking the ratio therefore gives \eqref{eq:tilted-score-limit-general}.

It remains to verify the integrability assumption for the two stated choices of $\varphi$. For $\varphi(x)=x$, the strict inequality $\norm[2]{\xi}<\varepsilon$ gives
\begin{align*}
    \int_{\R^d}e^{\xi^\top x}\rho(x)\norm[2]{x}\,\dd x
    &\le
    \int_{\R^d}e^{\norm[2]{\xi}\norm[2]{x}}\rho(x)\norm[2]{x}\,\dd x
    \le
    \frac{1}{\varepsilon-\norm[2]{\xi}}
    \int_{\R^d}e^{\varepsilon\norm[2]{x}}\rho(x)\,\dd x
    <\infty,
\end{align*}
where we used $r e^{-(\varepsilon-\norm[2]{\xi})r}\le(\varepsilon-\norm[2]{\xi})^{-1}$ for $r\ge0$. For $\varphi(x)=\tau\nabla Q(x)$,
\[
    \int_{\R^d}e^{\xi^\top x}\rho(x)\norm[2]{\tau\nabla Q(x)}\,\dd x
    \le
    \int_{\R^d}e^{\varepsilon\norm[2]{x}}\rho(x)\norm[2]{\tau\nabla Q(x)}\,\dd x
    <\infty.
\]
Both bounds follow from \eqref{eq:weighted-exponential-moment}. Applying \eqref{eq:tilted-score-limit-general} to these two continuous functions proves \eqref{eq:tilted-score-limit-X1-and-nablaQ}. Since their limits are finite, both posterior means are bounded for all sufficiently small $t$.
\end{proof}

\restatetheorem{\scfVelocityOptimality}

\begin{proof}
Suppose $\rho^v=\pi$ and $v=v^\pi$. By \eqref{eq:target-identity} and \eqref{eq:score-calibration-operator},
\[
    v_t^\pi(x_t)=\frac{\dot\alpha_t}{\alpha_t}x_t
    +\frac{\kappa_t}{\alpha_t}\E_{X_1\sim p_{1|t}^{\pi}(\cdot|x_t)}[\tau\nabla Q(X_1)]
    =(\iter v^\pi)_t(x_t)
\]
for every $t\in(0,1)$ and $x_t\in\R^d$. Therefore, $v^\pi=\iter v^\pi$. It remains to prove the converse direction.

Assume $v=\iter v$. For notational simplicity, denote $\rho:=\rho^v$ and $c_t:=\frac{\beta_t^2}{\alpha_t}$. Fix $\xi\in\R^d$ with $\norm[2]{\xi}<\min\left\{\delta,\frac{\varepsilon}2\right\}$, where $\delta$ and $\varepsilon$ are from \eqref{eq:source-growth}--\eqref{eq:weighted-exponential-moment} in Appendix~\ref{app:regularity}. By Proposition~\ref{prop:tilted-posterior-limit}, the limit
\begin{equation}\label{eq:limit_existence}
    \lim_{t\downarrow0}\left(\xi+\E_{X_1\sim p_{1|t}^{\rho}(\cdot|c_t\xi)}[\tau\nabla Q(X_1)]\right)
    =
    \frac{\int_{\R^d}e^{\xi^\top x}\rho(x)\left(\xi+\tau\nabla Q(x)\right)\,\dd x}
    {\int_{\R^d}e^{\xi^\top x}\rho(x)\,\dd x}\in\R^d
\end{equation}
exists. We first prove that this limit is zero, then apply Proposition~\ref{prop:exponential-tilt-identification} to identify $\rho$, and finally use the fixed-point equation to identify $v$.

The assumption $v=\iter v$ gives
\[
    v_t(c_t\xi)=\frac{\dot\alpha_t\beta_t^2}{\alpha_t^2}\xi
    +\frac{\kappa_t}{\alpha_t}\E_{X_1\sim p_{1|t}^{\rho}(\cdot|c_t\xi)}[\tau\nabla Q(X_1)].
\]
Using \eqref{eq:kappa}, we obtain
\begin{align*}
    \frac{\alpha_t}{\beta_t^2}v_t(c_t\xi)-\frac{\dot\beta_t}{\beta_t}\xi
    =&\left(\frac{\dot\alpha_t}{\alpha_t}-\frac{\dot\beta_t}{\beta_t}\right)\xi
    +\frac{\kappa_t}{\beta_t^2}\E_{X_1\sim p_{1|t}^{\rho}(\cdot|c_t\xi)}[\tau\nabla Q(X_1)]\\
    =&\frac{\kappa_t}{\beta_t^2}\left(\xi+\E_{X_1\sim p_{1|t}^{\rho}(\cdot|c_t\xi)}[\tau\nabla Q(X_1)]\right).
\end{align*}
Since $t_0<1$, \eqref{eq:source-growth} implies
\begin{align}
    &\int_0^{t_0}\frac{\kappa_t}{\beta_t^2}\norm[2]{\xi+\E_{X_1\sim p_{1|t}^{\rho}(\cdot|c_t\xi)}[\tau\nabla Q(X_1)]}\,\dd t\notag\\
    =&\int_0^{t_0}\norm[2]{\frac{\alpha_t}{\beta_t^2}v_t(c_t\xi)-\frac{\dot\beta_t}{\beta_t}\xi}\,\dd t\notag\\
    \le&\int_0^{t_0}\left(\frac{\alpha_t}{\beta_t^2}\norm[2]{v_t(c_t\xi)}+\abs{\frac{\dot\beta_t}{\beta_t}}\norm[2]{\xi}\right)\,\dd t\notag\\
    \le&\int_0^{t_0}\left(\frac{\alpha_t}{\beta_t^2}\left(a_t+b_tc_t\norm[2]{\xi}\right)+\abs{\frac{\dot\beta_t}{\beta_t}}\norm[2]{\xi}\right)\,\dd t\notag\\
    =&\int_0^{t_0}\frac{\alpha_t}{\beta_t^2}a_t\,\dd t
    +\norm[2]{\xi}\int_0^{t_0}\left(b_t+\abs{\frac{\dot\beta_t}{\beta_t}}\right)\,\dd t<\infty.
    \label{eq:integrable-tilted-discrepancy}
\end{align}
On the other hand, \eqref{eq:kappa} implies
    \begin{equation*}
        \int_0^{t_0}\frac{\kappa_t}{\beta_t^2}\,\dd t
        =\left.\log\frac{\alpha_{t}}{\beta_{t}}\right|_{t=0}^{t_0}
        =\infty.
    \end{equation*}

Suppose, for contradiction, that the limit in \eqref{eq:limit_existence} is nonzero. Then there exist $m>0$ and $t_1\in(0,t_0)$ such that
\[
    \norm[2]{\xi+\E_{X_1\sim p_{1|t}^{\rho}(\cdot|c_t\xi)}[\tau\nabla Q(X_1)]}\ge m
    \quad\text{for every }t\in(0,t_1).
\]
Consequently, the integral in \eqref{eq:integrable-tilted-discrepancy} is at least $m\int_0^{t_1}\frac{\kappa_t}{\beta_t^2}\,\dd t=\infty$, which leads to a contradiction. Therefore, the limit \eqref{eq:limit_existence} is zero. Multiplying by the strictly positive limiting denominator in \eqref{eq:limit_existence}, and using the arbitrariness of $\xi$, gives
\begin{equation}\label{eq:exponential-score-identities}
    \int_{\R^d}e^{\xi^\top x}\rho(x)\left(\xi+\tau\nabla Q(x)\right)\,\dd x=0
    \quad\text{for every }\norm[2]{\xi}<\min\left\{\delta,\frac{\varepsilon}2\right\}.
\end{equation}

Applying Proposition~\ref{prop:exponential-tilt-identification} with \eqref{eq:weighted-exponential-moment} and \eqref{eq:exponential-score-identities} gives $\rho=\pi$. Substituting this equality into $v=\iter v$ and using \eqref{eq:target-identity} yields $v=v^\pi$.
\end{proof}

\restatecorollary{\scfGeneralizedOptimality}

\begin{proof}
    For notational simplicity, denote $\rho:=\rho^v$ and $c_t:=\frac{\beta_t^2}{\alpha_t}$. Since $\lambda_t+\zeta_t>0$, denote $\bar{\lambda}_t:=\frac{\lambda_t}{\lambda_t+\zeta_t}\in(0,1]$. Then, $v=\iter_{\lambda,\zeta}v$ is equivalent to
    \begin{align}
        v_t(x_t)&=\left(1-\bar{\lambda}_t\right)
        \E_{X_1\sim p_{1|t}^{\rho}(\cdot|x_t)}[v_{t|1}(x_t|X_1)]+\bar{\lambda}_t\left(
        \frac{\dot\alpha_t}{\alpha_t}x_t
        +\frac{\kappa_t}{\alpha_t}\E_{X_1\sim p_{1|t}^{\rho}(\cdot|x_t)}[\tau\nabla Q(X_1)]\right).
        \label{eq:generalized-fixed-point-expanded}
    \end{align}

    Suppose $v=v^\pi$ and $\rho=\pi$. Then, $v_t^\pi(x_t)=\E_{X_1\sim p_{1|t}^{\pi}(\cdot|x_t)}[v_{t|1}(x_t|X_1)]$ and $v^\pi=\iter v^\pi$ imply \eqref{eq:generalized-fixed-point-expanded}, which gives $v^\pi=\iter_{\lambda,\zeta}v^\pi$.
    
    Suppose $v=\iter_{\lambda,\zeta}v$, which gives \eqref{eq:generalized-fixed-point-expanded}. In the following, we prove $v=v^\pi$ and $\rho=\pi$ following the proof of Theorem \ref{thm:velocity-optimality}. Fix $\xi\in\R^d$ with $\norm[2]{\xi}<\min\left\{\delta,\frac\varepsilon2\right\}$, where $\delta$ and $\varepsilon$ are from \eqref{eq:source-growth}--\eqref{eq:weighted-exponential-moment} in Appendix~\ref{app:regularity}. By Proposition~\ref{prop:tilted-posterior-limit}, the limit
    \begin{equation}\label{eq:limit_existence_cor_3_5}
        \lim_{t\downarrow0}\left(\xi+\E_{X_1\sim p_{1|t}^{\rho}(\cdot|c_t\xi)}[\tau\nabla Q(X_1)]\right)
        =
        \frac{\int_{\R^d}e^{\xi^\top x}\rho(x)\left(\xi+\tau\nabla Q(x)\right)\,\dd x}
        {\int_{\R^d}e^{\xi^\top x}\rho(x)\,\dd x}\in\R^d
    \end{equation}
    exists. We first prove that this limit is zero, then apply Proposition~\ref{prop:exponential-tilt-identification} to identify $\rho$, and finally use the fixed-point equation to identify $v$.

The assumption $v=\iter_{\lambda,\zeta}v$ gives
\begin{align*}
    &\frac{\alpha_t}{\beta_t^2}v_t\left(c_t\xi\right)
    -\frac{\dot\beta_t}{\beta_t}\xi=\left(1-\bar{\lambda}_t\right)\frac{\alpha_t^2\kappa_t}{\beta_t^4}\E_{X_1\sim p_{1|t}^{\rho}(\cdot|c_t\xi)}[X_1]
    +\bar{\lambda}_t\frac{\kappa_t}{\beta_t^2}\left(\xi+\E_{X_1\sim p_{1|t}^{\rho}(\cdot|c_t\xi)}[\tau\nabla Q(X_1)]\right).
\end{align*}
Following \eqref{eq:integrable-tilted-discrepancy}, the condition \eqref{eq:source-growth} implies
\begin{align*}
    \int_0^{t_0}\norm[2]{\frac{\alpha_t}{\beta_t^2}v_t(c_t\xi)-\frac{\dot\beta_t}{\beta_t}\xi}\,\dd t<\infty.
\end{align*}
Note that
\begin{equation*}
    \frac{\alpha_t^2\kappa_t}{\beta_t^4}
    =\frac12\frac{\dd}{\dd t}\left(\frac{\alpha_t^2}{\beta_t^2}\right),
    \qquad
    \int_0^{t_0}\frac{\alpha_t^2\kappa_t}{\beta_t^4}\,\dd t
    =\frac{\alpha_{t_0}^2}{2\beta_{t_0}^2}<\infty.
\end{equation*}
By Proposition~\ref{prop:tilted-posterior-limit}, the existence of the limit \eqref{eq:tilted-score-limit-X1} implies
\[
    \int_0^{t_0}\norm[2]{\left(1-\bar{\lambda}_t\right)\frac{\alpha_t^2\kappa_t}{\beta_t^4}\E_{X_1\sim p_{1|t}^{\rho}(\cdot|c_t\xi)}[X_1]}\,\dd t<\infty.
\]
Therefore, 
\begin{equation}
    \int_0^{t_0}\norm[2]{\bar{\lambda}_t\frac{\kappa_t}{\beta_t^2}\left(\xi+\E_{X_1\sim p_{1|t}^{\rho}(\cdot|c_t\xi)}[\tau\nabla Q(X_1)]\right)}\,\dd t<\infty.
    \label{eq:integrable-tilted-discrepancy-generalized}
\end{equation}

Suppose, for contradiction, that the limit in \eqref{eq:limit_existence_cor_3_5} is nonzero. Then there exist $m>0$ and $t_1\in(0,t_0)$ such that
\[
    \norm[2]{\xi+\E_{X_1\sim p_{1|t}^{\rho}(\cdot|c_t\xi)}[\tau\nabla Q(X_1)]}\ge m
    \quad\text{for every }t\in(0,t_1).
\]
Since $0<\bar{\lambda}_t\le1$ and $\frac{\kappa_t}{\beta_t^2}$ is continuous on $[t_1,t_0]$, condition \eqref{eq:generalized-uniqueness-condition} implies
\[
    \int_0^{t_1}\bar{\lambda}_t\frac{\kappa_t}{\beta_t^2}\,\dd t=\infty.
\]
Consequently, the integral in \eqref{eq:integrable-tilted-discrepancy-generalized} is at least $m\int_0^{t_1}\bar{\lambda}_t\frac{\kappa_t}{\beta_t^2}\,\dd t=\infty$, which leads to a contradiction. Therefore, the limit \eqref{eq:limit_existence_cor_3_5} is zero. Multiplying by the strictly positive limiting denominator in \eqref{eq:limit_existence_cor_3_5}, and using the arbitrariness of $\xi$, gives
\begin{equation*}
    \int_{\R^d}e^{\xi^\top x}\rho(x)\left(\xi+\tau\nabla Q(x)\right)\,\dd x=0
    \quad\text{for every }\norm[2]{\xi}<\min\left\{\delta,\frac{\varepsilon}2\right\}.
\end{equation*}
Applying Proposition~\ref{prop:exponential-tilt-identification} gives $\rho=\pi$. By \eqref{eq:population-velocity} and \eqref{eq:target-identity}, both posterior expressions in \eqref{eq:generalized-fixed-point-expanded} then equal $v_t^\pi(x_t)$. Hence,
\[
    v_t(x_t)=(1-\bar{\lambda}_t)v_t^\pi(x_t)+\bar{\lambda}_t v_t^\pi(x_t)=v_t^\pi(x_t)
\]
for every $t\in(0,1)$ and $x_t\in\R^d$.
\end{proof}

\begin{restatable}{proposition}{scfGlobalGrowth}
\label{prop:global-growth}
Let $\pi$ be the target density defined in \eqref{eq:target-density}, and assume that the regularity condition \eqref{eq:posterior-moment-condition} holds for $\rho=\pi$. Assume that $v^\pi$ and $v$ are admissible, as specified in Appendix~\ref{app:regularity}. Suppose that there exist nonnegative measurable functions $a,b\in L^1(0,1)$ such that
\begin{equation}
    \norm[2]{v_t(x_t)}\le a_t+b_t\norm[2]{x_t}\quad\text{for almost every }t\in(0,1)\text{ and every }x_t\in\R^d.
    \label{eq:source-growth-t01}
\end{equation}
Assume additionally that there exist constants $C_Q\in[0,\infty)$ and $m\ge0$ such that
\begin{equation}
    \norm[2]{\tau\nabla Q(x)}\le C_Q(1+\norm[2]{x}^{m})
    \qquad\text{for every }x\in\R^d.
    \label{eq:polynomial-score}
\end{equation}
Then \eqref{eq:source-growth} holds for every $t_0\in(0,1)$, and \eqref{eq:weighted-exponential-moment} holds for every $\varepsilon>0$. Consequently, $v=\iter v$ holds on $(0,1)\times\R^d$ if and only if $\rho^v=\pi$ and $v=v^\pi$.
\end{restatable}

\begin{proof}
Based on Theorem \ref{thm:velocity-optimality}, it remains to verify conditions \eqref{eq:source-growth}--\eqref{eq:weighted-exponential-moment}.

For every $t_0\in(0,1)$, $\lim_{t\downarrow0}\frac{\alpha_t}{\beta_t^2}=0$ and $\frac{\alpha_t}{\beta_t^2}$ is bounded on $[0,t_0]$; hence, $a\in L^1(0,1)$ implies $\frac{\alpha_t}{\beta_t^2}a_t\in L^1(0,t_0)$. Therefore, \eqref{eq:source-growth-t01} implies \eqref{eq:source-growth}.

Along an ODE trajectory $Y_t$ of $v$, the following inequality holds for almost every $t\in(0,1)$:
\[
\frac{\dd}{\dd t}\norm[2]{Y_t}\leq \norm[2]{v_t(Y_t)}\leq a_t+b_t\norm[2]{Y_t}.
\]
Denote $\norm[L^1]{a}=\int_0^1\abs{a_t}\,\dd t$ and $\norm[L^1]{b}=\int_0^1\abs{b_t}\,\dd t$. Gr\"onwall's inequality therefore yields
\begin{equation*}
    \norm[2]{Y_1}\le
    \exp \left(\int_0^1b_t\,\dd t\right)
    \left(\norm[2]{Y_0}+\int_0^1a_t\,\dd t\right)
    \leq e^{\norm[L^1]{b}}\left(\norm[2]{Y_0}+\norm[L^1]{a}\right).
\end{equation*}
Therefore, for every $\varepsilon>0$, \eqref{eq:polynomial-score} implies
\begin{align*}
    &\E_{Y_1\sim\rho^v}\left[e^{\varepsilon\norm[2]{Y_1}}\left(1+\norm[2]{\tau\nabla Q(Y_1)}\right)\right]\\
    \leq&\E_{Y_1\sim\rho^v}\left[e^{\varepsilon\norm[2]{Y_1}}\left(1+C_Q(1+\norm[2]{Y_1}^{m})\right)\right]\\
    \leq&\E_{Y_0\sim p_0}\left[\exp\left(\varepsilon e^{\norm[L^1]{b}}\left(\norm[2]{Y_0}+\norm[L^1]{a}\right)\right)\left(1+C_Q+C_Qe^{m\norm[L^1]{b}}\left(\norm[2]{Y_0}+\norm[L^1]{a}\right)^m\right)\right]\\
    <&\infty,
\end{align*}
which gives \eqref{eq:weighted-exponential-moment}.

According to Theorem \ref{thm:velocity-optimality}, $v=\iter v$ holds on $(0,1)\times\R^d$ if and only if $\rho^v=\pi$ and $v=v^\pi$.
\end{proof}

\subsection{Coefficient Selection}
\label{app:coefficient-proofs}

The generalized velocity optimality condition and training objective leave the coefficient schedules unspecified. We select them to control numerical singularities in the regression target while preserving the uniqueness guarantee.

\paragraph{Singularity cancellation.}
The score-based contribution in \eqref{eq:scf-target} contains the factors $\lambda_t\dot\alpha_t/\alpha_t$ and $\lambda_t\kappa_t/\alpha_t$. Using the definition of $\kappa_t$, the latter becomes
\begin{equation*}
    \lambda_t\frac{\kappa_t}{\alpha_t}
    =\frac{\lambda_t}{\alpha_t^2}\beta_t(\dot\alpha_t\beta_t-\alpha_t\dot\beta_t).
\end{equation*}
The inverse powers of $\alpha_t$ can cause numerical difficulties near $t=0$. Choosing $\lambda_t=\eta\alpha_t^2$, with $\eta>0$, cancels these factors in both coefficients.

\begin{restatable}{proposition}{scfBoundedCalibration}
\label{prop:bounded-calibration}
Fix $\eta>0$. For $\lambda_t=\eta\alpha_t^2$,
\begin{equation}
    \lambda_t\frac{\dot\alpha_t}{\alpha_t}=\eta\alpha_t\dot\alpha_t,
    \qquad
    \lambda_t\frac{\kappa_t}{\alpha_t}=\eta\beta_t(\dot\alpha_t\beta_t-\alpha_t\dot\beta_t).
    \label{eq:bounded-calibration-coefficients}
\end{equation}
The two coefficients in \eqref{eq:bounded-calibration-coefficients} are continuous and bounded on $[0,1]$.
\end{restatable}

\begin{proof}
    Substituting $\lambda_t=\eta\alpha_t^2$ and the definition \eqref{eq:kappa} gives \eqref{eq:bounded-calibration-coefficients}. Since $\alpha,\beta\in C^1([0,1])$, the expressions on the right-hand sides extend continuously to $[0,1]$ and are bounded.
\end{proof}

\paragraph{Preserving fixed-point uniqueness.}
Theorem~\ref{thm:generalized-optimality} requires the integral condition \eqref{eq:generalized-uniqueness-condition}. The following proposition gives a sufficient condition.

\begin{restatable}{proposition}{scfGaussianCalibration}
\label{prop:gaussian-calibration}
Adopt the assumptions of Theorem~\ref{thm:velocity-optimality} on $v$ and $\pi$. Let $\lambda_t>0$ and $\zeta_t\ge0$ be measurable, finite schedules on $(0,1)$. If there exists a constant $C\in(0,\infty)$ such that $\zeta_t\le C\lambda_t$ for all sufficiently small $t>0$, then \eqref{eq:generalized-uniqueness-condition} holds. Consequently, $v=\iter_{\lambda,\zeta}v$ holds on $(0,1)\times\R^d$ if and only if $\rho^v=\pi$ and $v=v^\pi$.
\end{restatable}

\begin{proof}
    For notational simplicity, denote $\bar{\lambda}_t:=\frac{\lambda_t}{\lambda_t+\zeta_t}$. Suppose $\zeta_t\le C\lambda_t$ for some finite $C\in(0,\infty)$ and every $t\in(0,t_1)$, where $t_1\in(0,t_0)$. Then, $\bar{\lambda}_t\geq\frac{1}{1+C}>0$ for $t\in(0,t_1)$. Therefore,
    $$
    \int_0^{t_0}\bar{\lambda}_t\frac{\kappa_t}{\beta_t^2}\,\dd t\geq \int_0^{t_1}\bar{\lambda}_t\frac{\kappa_t}{\beta_t^2}\,\dd t\geq \frac{1}{1+C}\int_0^{t_1}\frac{\kappa_t}{\beta_t^2}\,\dd t=\infty,
    $$
    which gives \eqref{eq:generalized-uniqueness-condition}. Hence, under the assumptions of Theorem~\ref{thm:generalized-optimality}, $v=\iter_{\lambda,\zeta}v$ holds on $(0,1)\times\R^d$ if and only if $\rho^v=\pi$ and $v=v^\pi$.
\end{proof}

We impose $\zeta_t\le C\lambda_t$, with the following choice:
\begin{equation}
    \lambda_t=\eta\alpha_t^2,\qquad
    \zeta_t=\eta\min\{C\alpha_t^2,\beta_t^2\},
    \qquad \eta>0.
    \label{eq:scf-coefficients}
\end{equation}
Proposition~\ref{prop:gaussian-calibration} ensures the uniqueness guarantee of the velocity optimality condition.
    
\subsection{Bounded Domains}
\label{app:bounded-action-proof}

\restateproposition{\scfBoundedActions}

\begin{proof}
The coordinatewise map $x\mapsto\tanh x$ is a continuously differentiable bijection from $\R^d$ onto $(-1,1)^d$, with inverse $\operatorname{arctanh}$. Its Jacobian is diagonal with positive entries $1-\tanh^2x_i$. The change-of-variables formula therefore gives \eqref{eq:latent-density}, including the normalizing constant. Taking logarithms,
\begin{equation*}
    \log\widetilde{\pi}(x)=\tau Q(\tanh x)+\sum_{i=1}^d\log(1-\tanh^2x_i)
    -\log\int_{(-1,1)^d}\exp(\tau Q(a))\,\dd a.
\end{equation*}
The final term is independent of $x$. The gradient of the first term is
\begin{equation*}
    \nabla_x[\tau Q(\tanh x)]
    =\tau\operatorname{diag}(1-\tanh^2x)\where{\nabla_aQ(a)}{a=\tanh x}.
\end{equation*}
For each coordinate, differentiating the log-Jacobian term gives
\begin{equation*}
    \frac{\dd}{\dd x_i}\log(1-\tanh^2x_i)
    =\frac{-2\tanh x_i(1-\tanh^2x_i)}{1-\tanh^2x_i}=-2\tanh x_i.
\end{equation*}
Adding these gradients proves \eqref{eq:latent-score}. In particular, the Jacobian correction is not multiplied by the inverse temperature.

Finally, if $Y_1$ has density $\widetilde{\pi}$, applying the change-of-variables formula to $\tanh Y_1$ cancels the Jacobian factor in \eqref{eq:latent-density}, yielding density $\pi$. Conversely, if $\tanh Y_1$ has density $\pi$, applying $\operatorname{arctanh}$ uniquely determines its preimage density to be $\widetilde{\pi}$. This proves both directions of the distributional equivalence. The same argument applies at each fixed state when $Q(a)$ is replaced by $Q(s,a)$.
\end{proof}

\section{Experiment Details and Results}
\label{app:experiment-details}

\subsection{Toy Example}
\label{app:toy-example}

\paragraph{Target density.}
The target is $\pi(x)\propto\exp(\tau Q(x))$ on $\R^2$ with $\tau=1$ and
\begin{equation*}
    Q(x)=\log\left[\frac18\sum_{k=0}^{7}\exp\left(-\frac12\norm[2]{x-\mu_k}^2\right)\right],
    \qquad
    \mu_k=\begin{pmatrix}8\cos(2\pi k/8)\\8\sin(2\pi k/8)\end{pmatrix},
\end{equation*}
an equally weighted mixture of eight unit-variance Gaussians whose means lie on a circle.

\paragraph{Settings.}
All five methods are integrated into one PyTorch codebase and share the network, optimizer, batch size, training budget, and evaluation protocol to ensure a fair comparison. During training, every algorithm has access only to the target score $\tau\nabla Q(x)$; ground-truth samples are used for evaluation only. All methods are given the same budget of network updates and target-score queries. We use a time-conditioned multilayer perceptron with $4$ hidden layers of $128$ units, Adam with learning rate $5\times10^{-4}$, and batch size $512$. AS, ASBS, and BMS are implemented following their open-source implementations, and FS following its paper. SCF uses $\alpha_t=t$, $\beta_t=1-t$, and the coefficients \eqref{eq:scf-coefficients} with $\eta=1$ and $C=10^4$. For ASBS, the update budget is split equally between its corrector matching and adjoint matching.
For sampling, SCF uses $50$ NFEs, while the baselines use $100$ NFEs.
All runs are conducted on a desktop with an NVIDIA GeForce RTX 5090 GPU and an Intel Core Ultra 9 285K CPU.

\begin{algorithm}[t]
\caption{Online Reinforcement Learning with Score-Calibrated Flow}
\label{alg:scf}
\DontPrintSemicolon
\KwIn{network parameters $\theta,\omega_1,\omega_2$;
    target network parameters
    $\bar{\omega}_1\leftarrow\omega_1,\bar{\omega}_2\leftarrow\omega_2$;
    inverse temperature $\tau$;
    learning rates $\ell_{\text{actor}},\ell_{\text{critic}}$;
    target smoothing coefficient $\nu$;
    replay buffer $\mathcal D\leftarrow\varnothing$.}
\For{each iteration}{
    \For{each environment step}{
        Sample action $a\sim\pi^\theta(\cdot|s)$
        and execute it in the environment;\;
        Observe next state $s'$ and reward $r$;\;
        Store transition
        $\mathcal D\leftarrow\mathcal D\cup\{(s,a,r,s')\}$.\;
    }
    \For{each gradient step}{
        Sample a mini-batch from $\mathcal D$;\;
        Update critics
        $\omega_i\leftarrow\omega_i-\ell_{\text{critic}}\nabla_{\omega_i}\calL_Q(\omega_i)$
        for $i\in\{1,2\}$;\;
        Update actor
        $\theta\leftarrow\theta-\ell_{\text{actor}}\nabla_\theta
        \calL_{\lambda,\zeta}^{\mathrm{actor}}(\theta)$;\;
        Update target networks
        $\bar{\omega}_i\leftarrow\nu\omega_i+(1-\nu)\bar{\omega}_i$
        for $i\in\{1,2\}$.\;
    }
}
\end{algorithm}

\begin{table}[t]
    \caption{Shared settings of the RL experiments.}
    \label{tab:rl-shared}
    \centering
    \small
    \begin{tabular}{lcc}
        \toprule
        & DeepMind Control Suite & HumanoidBench \\
        \midrule
        Parallel environments & $1$ & $50$ \\
        Environment steps per training step & $1$ & $50$ \\
        Replay buffer size & $2.5\times10^5$ & $1\times10^6$ \\
        Batch size & $256$ & $256$ \\
        Hidden layers & $2\times256$ & $3\times256$ \\
        Normalization & none & batch renormalization \\
        Optimizer & Adam & Adam \\
        Actor learning rate & $3\times10^{-4}$ & $3\times10^{-4}$ \\
        Critic learning rate & $1\times10^{-3}$ & $3\times10^{-4}$ \\
        Discount factor & $0.99$ & $0.99$ \\
        Target smoothing coefficient & $0.005$ & $0.005$ \\
        \bottomrule
    \end{tabular}
\end{table}

\paragraph{Metrics.}
We evaluate $20$ checkpoints per run, evenly spaced in model updates, each with $5\times10^4$ generated samples. The checkpoint with the smallest $W_2$ is selected and re-evaluated on fresh generated and reference samples, and we report the mean and standard deviation of the resulting values over five seeds.
\begin{itemize}
    \item \textbf{Mode total variation distance (mode TVD).} Each generated sample is assigned to the component with the nearest mean, and $\widehat w_k$ denotes the fraction of samples assigned to component $k$. The mode TVD $\frac12\sum_{k=0}^{7}\abs{\widehat w_k-\tfrac18}$ measures how evenly the eight modes are covered.
    \item \textbf{$2$-Wasserstein distance ($W_2$).} Computed exactly under the squared Euclidean cost between $2000$ generated samples and $2000$ target samples; we repeat the computation $20$ times with nonoverlapping subsets of generated samples and fresh target samples, and report the median. As a finite-sample estimate it does not vanish even for perfect samples: between two independent sets of target samples it equals $0.719$, which serves as the reference floor.
\end{itemize}

\paragraph{Results.}

The results are summarized in Figure~\ref{fig:toy}. Black crosses mark the component means, the circles mark one and two standard deviations around each mean, and samples are colored by their nearest mean. Each panel reports the mode TVD, the $W_2$ distance, and the wall-clock training time as mean $\pm$ standard deviation over five seeds (lower is better). SCF achieves the lowest discrepancy on both mode TVD and $W_2$ while being the most efficient in terms of training and inference.

\subsection{RL Tasks}
\label{app:rl-tasks}

\paragraph{Shared settings.}
For each set of RL tasks, all algorithms are integrated into one JAX codebase and share the training loop, replay buffer, networks, and optimizer, summarized in Table~\ref{tab:rl-shared}, to ensure a fair comparison. On the DeepMind Control Suite we build on the codebase of RFM~\citep{li2026reverse} (\url{https://github.com/azizanlab/ReverseFlowMatching}); on HumanoidBench we build on the codebase of QFlex~\citep{wei2026scalable} (\url{https://github.com/LNSGroup/Qflex}). Each algorithm is implemented following its open-source implementation. All training and evaluation runs are conducted on a desktop with an NVIDIA GeForce RTX 5090 GPU and an Intel Core Ultra 9 285K CPU.

\paragraph{Method-specific settings.}
SCF uses the latent parameterization of Section~\ref{sec:bounded-domains} with $\alpha_t=t$, $\beta_t=1-t$, and the coefficients \eqref{eq:scf-coefficients} with $\eta=1$ and $C=10^4$. The overall procedure is summarized in Algorithm~\ref{alg:scf}.
For sampling, each baseline uses the NFEs of its original implementation: $10$ for RFM, $16$ for DIME, and $20$ for the others; SCF uses $10$. For evaluation, each baseline also follows its own choice of best-of-$n$ action selection, in which $n$ candidate actions are sampled and the one with the highest critic value is executed: RFM, QVPO, MaxEntDP, SDAC, and DIME use $n=32$, whereas DQS, QSM, and QFlex use $n=1$. SCF uses $n=32$ on the DeepMind Control Suite and $n=1$ on HumanoidBench.

\paragraph{Results.}

The training curves are shown in Figures~\ref{fig:hb} and~\ref{fig:dmc}. The wall-clock time per training iteration for each method is summarized in Table~\ref{tab:update-time}. The results demonstrate that SCF consistently ranks among the top-performing methods across all tasks, while also being the most efficient in the evaluated generative online RL algorithms.

\begin{figure}[t]
    \centering
    \includegraphics[width=\textwidth]{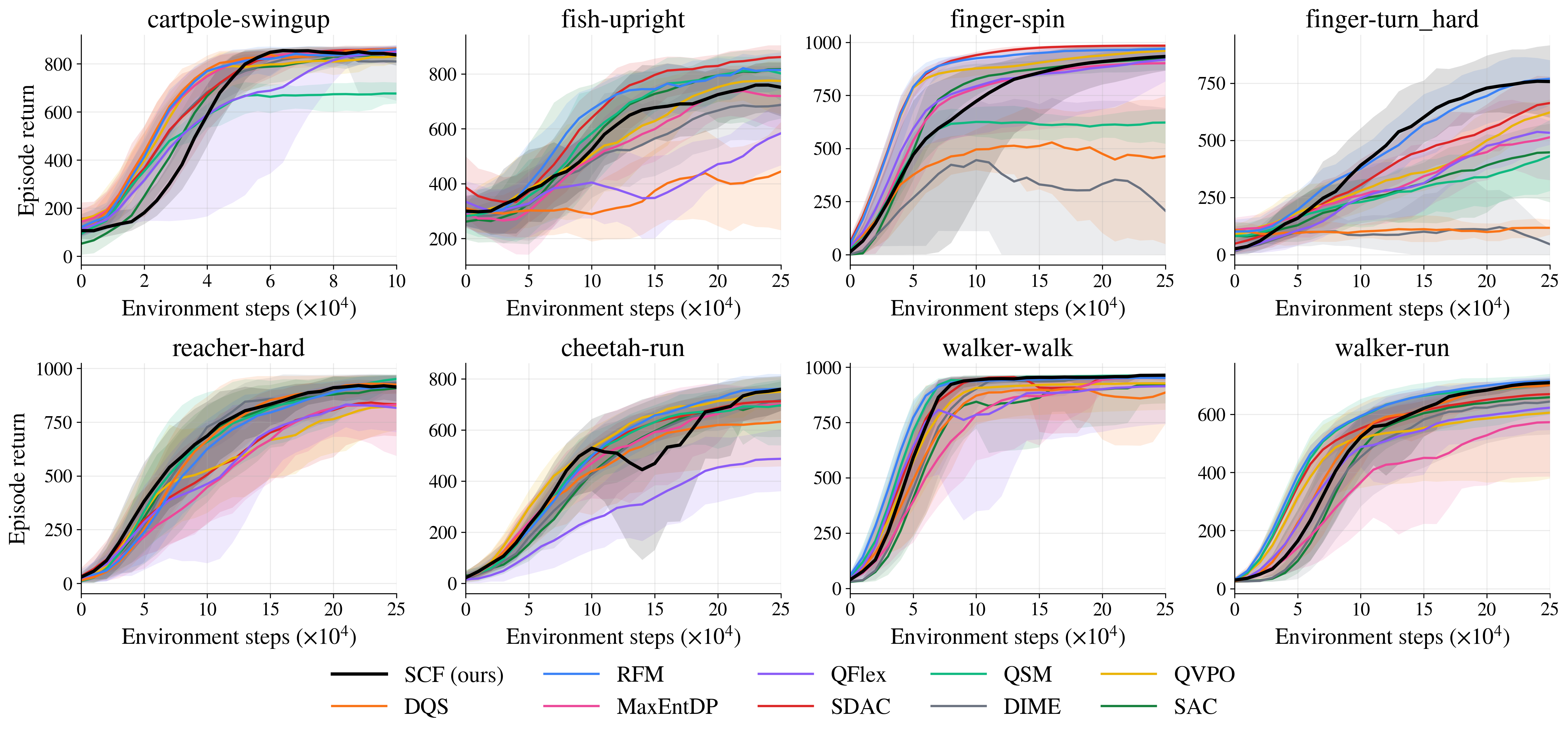}
    \caption{Training curves on eight tasks from the DeepMind Control Suite. The solid lines correspond to the mean across five seeds, and the shaded regions represent minimum-to-maximum range.}
    \label{fig:dmc}
\end{figure}

\begin{table}[t]
    \caption{Milliseconds per training iteration on walker-run (DeepMind Control Suite) and h1-run (HumanoidBench), measured on one NVIDIA GeForce RTX 5090 with batch size $256$ and averaged over $1000$ steps.}
    \label{tab:update-time}
    \centering
    \footnotesize
    \setlength{\tabcolsep}{4pt}
    \begin{tabular}{lcccccccccc}
        \toprule
        Task & SAC & SCF (ours) & QSM & RFM & QFlex & SDAC & DIME & QVPO & MaxEntDP & DQS \\
        \midrule
        walker-run & 0.47 & 0.87 & 1.08 & 2.67 & 1.94 & 2.43 & 3.24 & 4.40 & 6.11 & 7.51 \\
        h1-run & 1.02 & 1.12 & 1.21 & 2.08 & 2.99 & 3.15 & 3.33 & 4.58 & 6.28 & 11.66 \\
        \bottomrule
    \end{tabular}
\end{table}

\begin{figure}[t]
    \centering
    \includegraphics[width=0.4\textwidth]{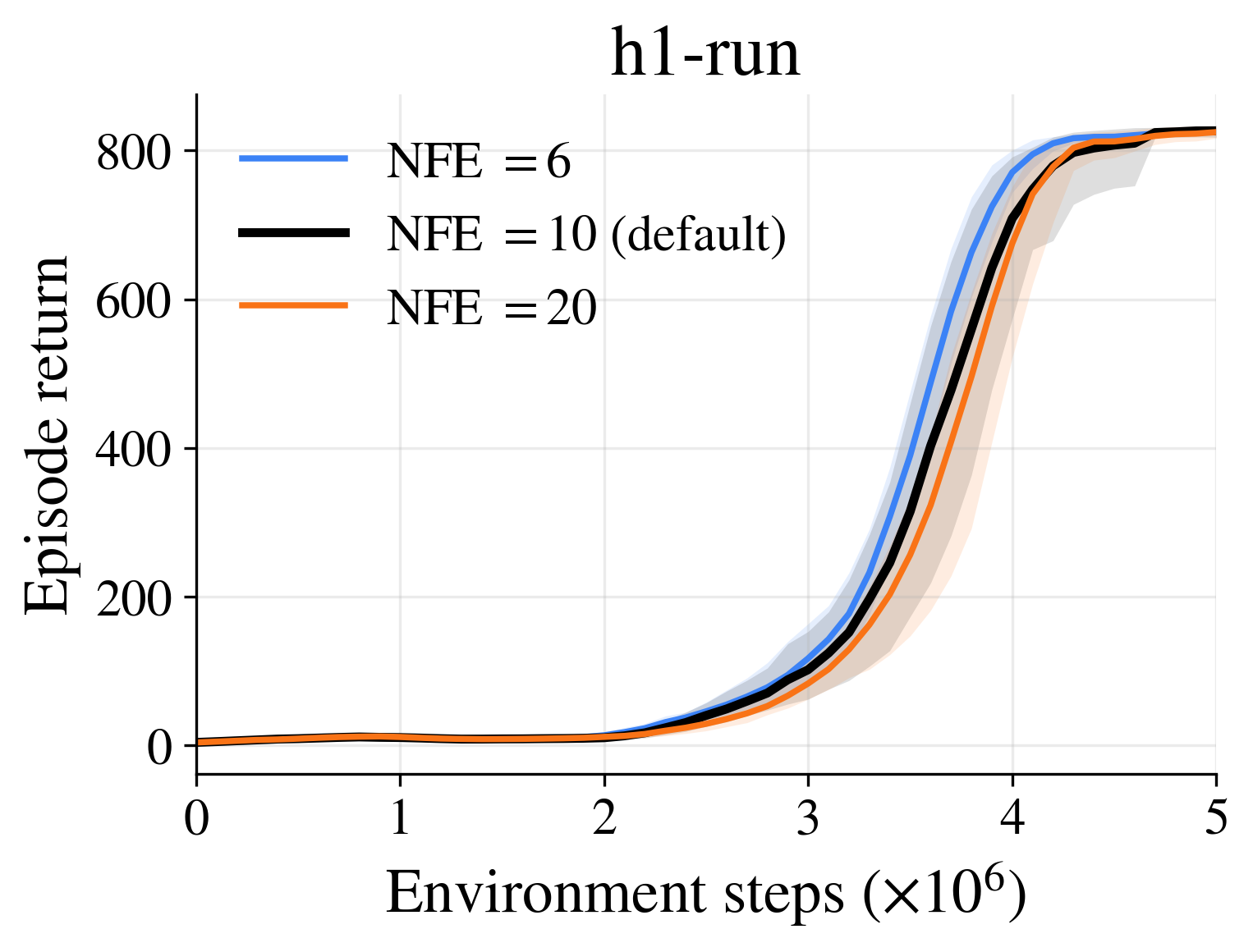}
    \qquad
    \includegraphics[width=0.4\textwidth]{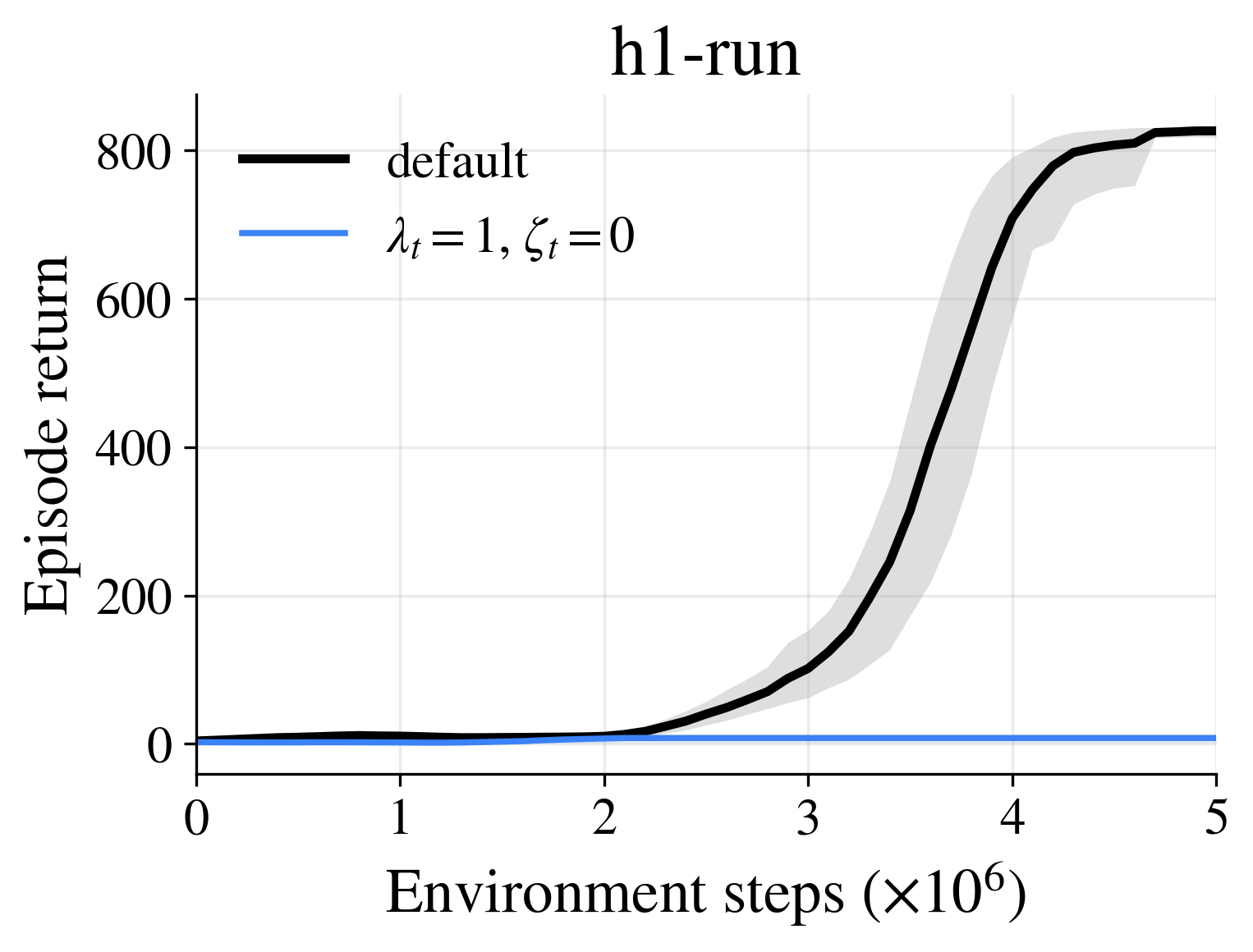}
    \caption{Training curves of SCF on the h1-run task under different design choices. Left: experiments on NFEs for sampling. Right: experiments on the coefficients of the regression target. The solid lines correspond to the mean across five seeds, and the shaded regions represent minimum-to-maximum range.}
    \label{fig:ablation}
\end{figure}

\section{Ablation Studies}
\label{app:ablation}

In this section, we analyze how different design choices affect the performance of SCF. We conduct the experiments on the h1-run task from HumanoidBench. All variants follow the settings of Appendix~\ref{app:rl-tasks} and are trained with five seeds; the training curves are shown in Figure~\ref{fig:ablation}.

First, we vary the NFEs used to sample actions during both training and evaluation. SCF uses $10$ NFEs by default; we compare with $6$ and $20$ NFEs. As shown in Figure~\ref{fig:ablation} (left), the three settings perform similarly, indicating that the algorithm is robust to the choice of NFEs.

Second, we study the coefficient choices in the regression target \eqref{eq:scf-target}. Our default choice \eqref{eq:scf-coefficients} combines the conditional velocity, current velocity, and score-based term, with $\lambda_t=\eta\alpha_t^2$ canceling the inverse powers of $\alpha_t$ in the score-based contribution (Proposition~\ref{prop:bounded-calibration}). As an ablation, we set $\lambda_t=1$ and $\zeta_t=0$, leaving only the score-based term. Since this term is unbounded as $t\to0$, we sample $t$ uniformly from $[0.05,1]$ to avoid the singularity. Nevertheless, this variant fails to learn: for every seed, the regression loss exceeds $10^{24}$ within the first $2\times10^6$ environment steps, and the return never improves (Figure~\ref{fig:ablation}, right). These results confirm that the singularity cancellation developed in Appendix~\ref{app:coefficient-proofs} is both theoretically well motivated and practically essential for stable training.

\end{document}